\documentclass[11pt]{article}

\usepackage{epsfig,amsmath,latexsym,amssymb}
\usepackage{graphicx}
\usepackage{lscape}
\usepackage{picture, eso-pic, tikz} 
\usepackage{dsfont}
\usepackage{natbib}
\usepackage{listings}
\usepackage{changes}
\usepackage{hyperref}
\usepackage{subcaption}
\usepackage{booktabs}

\usepackage{ulem}

\usepackage{pgfplots}
\pgfplotsset{compat=1.18} 
\usepackage{pgfplotstable}
\usepgfplotslibrary{groupplots}
\pgfplotsset{every tick label/.append style={font=\small}}
\usepgfplotslibrary{fillbetween}
\usepackage{subcaption}

\usepackage{tikzviolinplots}
\usepackage{minted}
\usepackage{scontents}

\def\R{{\mathbb R}}

\def\E{{\mathbb E}}

\def\Var{\operatorname{Var}}
\def\Beweis{\footnotesize}

\newcommand{\Remm}[1]{}
\newtheorem{theo}{Theorem}[section]

\newtheorem{proposition}[theo]{Proposition}

\newtheorem{cor}[theo]{Corollary}

\newtheorem{model ass}[theo]{Model Assumptions}

\newtheorem{example}[theo]{Example}

\makeatletter
\def\thm@space@setup{
  \thm@preskip=8pt plus 2pt minus 4pt
  \thm@postskip=\thm@preskip
}
\makeatother

\def\EndProof{\hfill {\scriptsize $\Box$}}

\numberwithin{equation}{section}

\definecolor{MyGray}{rgb}{0.92,0.92,0.92}


\definecolor{British racing}{rgb}{0.0, 0.5, 0.0}
\def\bx{\boldsymbol{x}}

\def\be{\boldsymbol{e}}

\def\bw{\boldsymbol{w}}

\def\bY{\boldsymbol{Y}}
\def\bX{\boldsymbol{X}}
\def\b0{\boldsymbol{0}}

\def\bbeta{\boldsymbol{\beta}}
\def\balpha{\boldsymbol{\alpha}}

\def\b0{\boldsymbol{0}}

\def\bff{\boldsymbol{f}}

\def\br{\boldsymbol{r}}

\def\bpi{\boldsymbol{\pi}}

\title{Insurance Pricing Optimization via Off-Policy Evaluation}
\author{Sascha G\"unther\footnote{Department of Mathematics, ETH Zurich, sascha.guenther@math.ethz.ch}
\and Dimitri Semenovich\footnote{dvsemenovich@gmail.com}
\and Mario V.~W\"uthrich\footnote{Department of Mathematics, ETH Zurich, mario.wuethrich@math.ethz.ch}}

\date{\today}
\begin{document}

\maketitle



\begin{abstract}
    Traditional insurance pricing relies on risk-based principles that ensure actuarial fairness and solvency but do not explicitly account for policyholders' price sensitivity. We formulate insurance pricing as a decision-making problem and study it using tools from off-policy evaluation and stochastic control. We propose a kernelized inverse propensity score estimator that exploits local structure in the action space and yields variance reduction compared to the classical inverse propensity score estimator. Building on these value estimates, we investigate policy optimization and present two practical approaches for computing optimal pricing rules: an interpretable data-shared Lasso formulation and a flexible policy parameterization based on neural networks.
    Using a controlled synthetic travel insurance environment, we empirically confirm the theoretical results and show that neural networks outperform existing techniques for policy optimization.
\end{abstract}

\section{Introduction}
Pricing is a central task in actuarial science. Traditionally, insurance prices are derived from expected values of future claims and expenses, adjusted by risk loadings and profit margins. This approach is fully risk-based, i.e., it focuses on covering expected costs rather than explicitly modeling the policyholder’s willingness to pay. While this framework ensures solvency and fairness, it may overlook important behavioral aspects of demand and pricing strategy optimization.

In recent years, dynamic and data-driven pricing approaches have gained attention in many industries. Firms in e-commerce, transportation, and energy sectors increasingly use reinforcement learning and counterfactual inference to optimize pricing decisions under uncertainty. These methods allow for learning from historical data collected under previous, possibly suboptimal, pricing policies, and estimating the performance of new pricing strategies without requiring new experiments. In actuarial science, these ideas have so far received limited attention. The dominant thread remains risk-based ratemaking relying on GLMs, GAMs, and other statistical methods. However, some articles do consider optimal or dynamic pricing when demand reacts to price. \cite{Krikler2004} describe an optimal pricing procedure that aims to maximize revenue and relies on demand estimations. \cite{Emms_Haberman_2005} use optimal control with demand modeling and show how bang-bang strategies can arise unless the pricing rule is smoothed by contract accrual. \cite{Emms2007} extends this to a stochastic market-average premium and solves the resulting Bellman equation to obtain dynamic premium strategies sensitive to loss-ratio dynamics. These papers explicitly integrate demand into the objective rather than only applying risk loadings to expected claims. Earlier, \cite{Moriconi1982} already argued that premium principles should account for market conditions, which is reflected in modern price-sensitive frameworks. On the industry side, \cite{Guven_McPhail_2013} popularized elasticity-based adjustments in property and casualty pricing, documenting industry interest.
Recent work frames insurance pricing as a contextual bandit or reinforcement learning problem. \cite{Krasheninnikova2019} use reinforcement learning techniques to maximize customer lifetime value. \cite{Treetanthiploet2023} optimize prices offered on price-comparison websites, while \cite{Young2024} use similar techniques to achieve a desired target portfolio in terms of a composition of risk factors such as age, location, and occupation.

A widely used approach in contextual pricing first estimates a demand function from historical data and then uses this estimate to determine profit-maximizing prices \citep[e.g.,][]{Ferreira2016, Baardman2023, Alley2023, Biggs2021, Chen2022}. In this framework, demand is modeled as the probability that a customer purchases a product as a function of the offered price and customer or product characteristics. An estimator of expected revenue is then constructed as the product of the price and the estimated purchase probability. Prices are ultimately chosen to maximize this estimated revenue. \cite{Elmachtoub2022} refer to this approach as \textit{predict-then-optimize}. The performance of this approach heavily depends on the accuracy of the estimated demand model. In practice, estimates are often biased due to model misspecification or selection effects induced by historical pricing policies \citep[see][]{Jagabathula2017, Semenovich2019}. This bias can propagate through the optimization step, leading to systematically suboptimal pricing decisions and inflated estimates of revenue. To mitigate these issues, \cite{Ye2018} and \cite{Biggs2022} replace the revenue objective with a tractable surrogate loss that directly targets pricing decisions. However, their approach is typically not based on unbiased estimates.

Outside of actuarial science, policy evaluation from historical data has been studied extensively.
This problem, known as off-policy evaluation, is central to reinforcement learning, where one seeks to estimate the expected reward of a policy $\pi$ using data generated under a potentially different policy $\widetilde{\pi}$. 
Early statistical approaches rely on inverse propensity weighting, tracing back to the estimator introduced by \cite{HorvitzThompson}, which provides unbiased but potentially high-variance estimates. 
\cite{Dudik} propose a doubly robust estimator that combines direct reward modeling with importance weighting to improve variance and robustness.
Subsequent work has refined these estimators to reduce the variance and to improve the finite-sample performance, for instance by employing a bias-variance tradeoff in \cite{Thomas2016} or by extending it to sequential decision-making problems in \cite{Jiang2016}.

This article transfers these techniques by proposing a general framework for optimal insurance pricing as a stochastic control problem. Each pricing decision is viewed as an action $A$ taken on a policyholder with characteristics $\bX$, leading to a reward $R$, for example, revenue or profit. The objective is to find a pricing policy $ \pi $ that maximizes the expected reward $ V(\pi) $. In doing so, we apply techniques from off-policy evaluation and policy optimization, which allow us to estimate and improve pricing policies from historical insurance data collected under past pricing schemes.

This paper makes four main contributions. First, while reinforcement learning approaches to insurance pricing have recently been proposed in dynamic and online settings, we formulate demand-sensitive insurance pricing on historical data as an off-policy control problem. A key practical advantage of this offline approach is the strict separation it enables between data collection and decision-making. During an experimentation phase, the insurer collects data under a known randomised policy $\widetilde{\pi}$. Once this data is gathered, off-policy estimators allow the insurer to evaluate an effectively unlimited number of candidate pricing strategies by reusing the same experimental data, i.e., each counterfactual policy evaluation requires no additional data collection. This makes the framework particularly attractive in insurance settings where experimentation is costly or operationally constrained. Moreover, because the estimators are computed offline, the insurer can conduct detailed analysis, compare alternative strategies, and validate results before committing to a new pricing rule in production.
Second, we introduce a kernelized inverse propensity score estimator that exploits local structure in the action space. The estimator smooths information across neighboring actions and allows policies defined on new action spaces to be evaluated using historical data. Under a mild regression assumption in the action variable, we establish unbiasedness of the estimator and derive variance bounds relative to the classical inverse propensity score estimator.
Third, we characterize the variance-optimal kernel matrix and show that this construction strictly improves upon the classical inverse propensity score estimator. We also propose a computationally efficient kernel construction that avoids estimating conditional reward moments while achieving comparable empirical performance.
Finally, we demonstrate how the proposed evaluation framework can be combined with both interpretable (data-shared Lasso) and flexible (neural network) policy parameterizations, and we compare these approaches to the classical predict-then-optimize method in a controlled synthetic insurance environment. In this setting, the kernelized inverse propensity score estimator substantially reduces variance compared to the classical inverse propensity score estimator, while the computationally efficient kernel construction performs nearly as well as the variance-optimal variant. In the policy optimization task, flexible neural network policies can achieve higher policy values, though at the cost of increased variability across simulations.

The remainder of the paper is structured as follows. In Section~\ref{sec:problem_formulation}, we introduce the stochastic control framework for insurance pricing and define the value of a pricing policy. Section~\ref{sec:estimation} discusses classical off-policy evaluation methods, namely the direct method and the inverse propensity score estimator. In Section~\ref{sec:kernelized_controls}, we extend these ideas by introducing kernelized controls, and we analyze their statistical properties, including their bias and variance behavior under a local regression assumption in the action variable. We turn to the optimization problem in Section~\ref{sec:policy_optimization} and present two approaches for computing optimal pricing policies, based on data-shared Lasso and neural network parameterizations. Finally, Section~\ref{sec:numerical_part} empirically verifies the theoretical results, and illustrates the proposed optimization methods in a synthetic travel insurance pricing environment and compares them to the classical predict-then-optimize approach.

\section{Problem formulation} \label{sec:problem_formulation}
Let $\bX \in {\cal X} \subseteq \R^p$ denote the features of the insurance policyholders, $A \in {\cal A}$ the actions taken (insurance prices charged), and $R$ the rewards (e.g., revenue) generated. We assume a finite action space ${\cal A}$ with $d$ levels.

\medskip

A stochastic control (policy) is given by a mapping
\begin{equation}\label{stochastic policy}
    \pi : {\cal A} \times {\cal X}\to [0,1]\,,
    \qquad (a,\bx) \mapsto \pi(a\mid\bx)\,,
\end{equation}
where $\pi(\cdot\mid\bx)$ is a probability distribution on the finite action space ${\cal A}$ for all $\bx \in {\cal X}$.
A selected policy $\pi$ has a value (average expected reward) of
\begin{equation}\label{value definition}
  V(\pi) = \E_{\pi} \left[ R \right] =\E \Big[ \E_{\pi}\left[ R \mid\bX \right] \Big]
  =\E \left[\, \sum_{a \in {\cal A}} \E\left[R \mid \bX, A=a \right] \pi(a\mid\bX) \right]\,.
\end{equation}
In an insurance pricing context, $V(\pi)$ represents the expected revenue obtained when applying the pricing rule $\pi$ across the portfolio.

The optimal controls are given by (subject to existence)
\begin{equation}\label{optimal control}
    \pi^* ~\in~ \underset{\pi}{\arg\max}~V(\pi)\,.
\end{equation}
This leads to two tasks: (i) computation/estimation of the value $V(\pi)$ under an arbitrary policy $\pi$, and (ii) finding an optimal control $\pi^*$. In Sections~\ref{sec:estimation} and~\ref{sec:kernelized_controls}, we will focus on the former before we turn to the latter in Section~\ref{sec:policy_optimization}.

\section{Estimation of value}\label{sec:estimation}
To compute (or estimate) the optimal control \eqref{optimal control}, we need to be able to evaluate (or approximate) the value $V(\pi)$ for any possible policy $\pi$. A main issue in solving this problem is that, per policyholder, typically, there is only data available from very few or even only one single policy. Let there be a tuple of features, action, and reward $(\bX_i, A_i, R_i)$ generated under given policies $\widetilde{\pi}_i$ for $i \in \{1, \dots, n\}$. We assume that
\begin{enumerate}
    \item[(i)] the tuples $(\bX_i, A_i, R_i)$,  $1\leq i\leq n$, are independent,
    \item[(ii)] the features of the policyholders $\bX_i$, $1\leq i\leq n$, are i.i.d.,
    \item[(iii)] given the features $\bX_i$, the distribution of the action $A_i$ is determined by the given policy $\widetilde{\pi}_i$, $1\leq i \leq n$,
    \item[(iv)] for any vector of features $\bx \in \mathcal{X}$ and any action $a \in \mathcal{A}$, the conditional distribution of the reward $R_i$, given $\bX_i = \bx$ and $A_i = a$, does not depend on the policyholder $i$, i.e.,
    \begin{equation*}
        R_i \mid_{\bX_i = \bx,\, A_i = a} \;\overset{d}{=}\; R_j \mid_{\bX_j = \bx,\, A_j = a} \quad \text{for all } i,j \in \{1, \dots, n\}\,.
    \end{equation*}
\end{enumerate}

We denote the collection of all (possibly different) past policies as $\widetilde{\bpi} = (\widetilde{\pi}_i)_{i=1}^n$ and let
\begin{equation}\label{observations assumed}
{\cal L}_{\widetilde{\bpi}}=\left( \bX_i, A_i, R_i\right)_{i=1}^n\,,
\end{equation}
be the observed learning sample.\\
Note that the policy $\widetilde{\pi}_i$ used to generate the learning sample may change from policyholder to policyholder, and thus, we can combine information collected under different past policies. The dependence of the policy $\widetilde{\pi}_i$ on the policyholder $i$, $1\leq i \leq n$, can either be interpreted as a simultaneous test of different policies to different policyholders or as the policy changing over time. This way, different policies can be applied to policyholders with identical features.

We present two different ways of off-policy learning to estimate $V(\pi)$ for an arbitrary policy $\pi$; see \cite{Dudik}. We revisit these two methods.

\bigskip

{\it (1) Direct method.} We approximate the value \eqref{value definition} by its empirical version
\begin{equation} \label{eq:empirical_value}
  \widehat{V}_n(\pi) = \frac{1}{n}\sum_{i=1}^n \sum_{a \in {\cal A}}\E\left[R \mid \bX = \bX_i, A=a \right] \pi(a\mid\bX_i)\,.
\end{equation}
In order to evaluate this empirical value, we need to compute (or estimate) the conditional expected return. For this, we set up a regression model
\begin{equation}\label{regression direct model}
  (\bx, a) ~\mapsto~ \varrho(\bx,a)=\E\left[R\mid \bX = \bx, A=a \right],
\end{equation}
and based on the learning sample ${\cal L}_{\widetilde{\bpi}}$, we can infer this regression model, resulting in an estimated regression function $(\bx, a)\mapsto \widehat{\varrho}(\bx,a)$ for all $\bx \in \mathcal{X}$ and $a\in \mathcal{A}$. In particular, this is justified by item (iv) of our assumptions above.

The {\it direct method} (DM) yields the following estimator for the policy value \eqref{value definition}:
\begin{equation}\label{DM estimate}
\widehat{V}_{\rm DM}(\pi) = \frac{1}{n}\sum_{i=1}^n \sum_{a \in {\cal A}}\widehat{\varrho}\left(\bX_i, a\right)\pi\left(a \mid \bX_i\right)\,.
\end{equation}
Note that the past policies $\widetilde{\pi}_i$, the observed rewards $R_i$ and the actions $A_i$, $1\leq i\leq n$, do not directly enter equation \eqref{DM estimate}, only implicitly via the sample $\mathcal{L}_{\widetilde{\bpi}}$ that was used to fit the regression model. A disadvantage of the direct method is that it requires specifying an explicit regression model class for \eqref{regression direct model} to solve the estimation problem. 

\bigskip

{\it (2) Inverse propensity score method.}
The second method uses inverse propensity scores, which
go back to \cite{HorvitzThompson}; in statistics and Bayesian modeling, this technique is also known as {\it importance weighting}. For a generic policy $\widetilde{\pi}$, we rewrite the value defined in \eqref{value definition} as follows 
\begin{eqnarray}
  \nonumber V(\pi)  &=&\E \left[\,\sum_{a \in {\cal A}} \E\left[R \mid \bX, A=a \right]\, \frac{\pi(a\mid\bX)}{\widetilde{\pi}(a\mid\bX)}\,\widetilde{\pi}(a\mid\bX) \right]  
  \\&=&\E \left[ \E_{\widetilde{\pi}}\left[R ~\frac{\pi(A\mid\bX)}{\widetilde{\pi}(A\mid\bX)} \,\middle|\, \bX \right] \right]~=~ \E_{\widetilde{\pi}} \left[ R ~\frac{\pi(A\mid\bX)}{\widetilde{\pi}(A\mid\bX)}\right]\,. \label{eq:importance_weighting}
\end{eqnarray}
The latter quantity is an expected value under an observed policy $\widetilde{\pi}$. Thus, we can use the learning sample ${\cal L}_{\widetilde{\bpi}}$ to estimate it empirically by the {\it inverse propensity score} (IPS) estimate
\begin{equation}\label{IPS estimate}
\widehat{V}_{\rm IPS}(\pi) = \frac{1}{n}\sum_{i=1}^n  R_i~\frac{\pi(A_i\mid\bX_i)}{\widetilde{\pi}_i(A_i\mid\bX_i)}\,.
\end{equation}
The advantage of the IPS estimate is that it is model-free, i.e., we do not need any model assumptions beyond having a learning sample ${\cal L}_{\widetilde{\bpi}}$. The disadvantage is that this estimator can have a high uncertainty through the variance of the importance weights $\pi(A_i\mid\bX_i)/\widetilde{\pi}_i(A_i\mid\bX_i)$. In the transformation \eqref{eq:importance_weighting} and the estimation in \eqref{IPS estimate}, we implicitly assume sufficient overlap between the policies $\widetilde{\pi}_i$ and $\pi$, meaning that $\widetilde{\pi}_i(a \mid \bx) > 0$ whenever $\pi(a \mid \bx) > 0$, $a\in\mathcal{A}$, $\bx\in\mathcal{X}$, so that the importance weights are well-defined.

\bigskip

The above has been formulated for a stochastic control
\eqref{stochastic policy}. In the following we turn our attention to {\it deterministic
controls} $\pi :{\cal X} \to {\cal A}$ with $a=\pi(\bx)\in {\cal A}$. That is, by a slight abuse of notation, we assume that $\pi$ maps every policyholder's characteristics $\bx$ (deterministically) to an action $a=\pi(\bx) \in {\cal A}$.

For a deterministic policy $\pi$, the IPS estimate is given by
\begin{equation}\label{IPS estimate 2}
    \widehat{V}_{\rm IPS}(\pi) 
    = \frac{1}{n}\sum_{i=1}^n R_i\, \frac{\mathds{1}_{\{A_i=\pi(\bX_i)\}}}{\widetilde{\pi}_i(A_i\mid\bX_i)}
    = \frac{1}{n}\sum_{i=1}^n R_i\, \frac{\langle\be_{A_i}, \be_{\pi(\bX_i)}\rangle}{\widetilde{\pi}_i(A_i\mid\bX_i)},
\end{equation}
with a one-hot encoded vector
\begin{eqnarray*}
\be :{\cal A}=\{a^{(1)}, \ldots,a^{(d)}\} &\to& \{0,1\}^d
\\ 
a &\mapsto& \be_a=\be(a)=\left( \mathds{1}_{\{a=a^{(1)}\}}, \ldots, 
 \mathds{1}_{\{a=a^{(d)}\}}\right)^\top.
 \end{eqnarray*}

\begin{proposition}\label{prop:unbiased_IPS}
The IPS estimator \eqref{IPS estimate 2} is unbiased for the value
$V(\pi)$.
\end{proposition}
{\Beweis
{\bf Proof.} This immediately follows from \eqref{eq:importance_weighting}.
\EndProof}

\section{Kernelized controls} \label{sec:kernelized_controls}

The above derivations were based on the assumption that both the observed policies  $\widetilde{\bpi}$ (used to collect the learning sample ${\cal L}_{\widetilde{\bpi}}$)  and the new policy $\pi$ of interest act on the same action space ${\cal A}$.  In addition, the IPS estimator \eqref{IPS estimate 2} relies only on observations for which the realized action matches the action prescribed by the new policy.  In practice, this can lead to high variance in the estimator, since much of the available data is discarded.

The following two considerations motivate extending this framework.

First, when considering pricing policies for future insurance contracts, the set of admissible prices may differ from those used historically. In that case, the new policy may act on a different action space $\bar{\mathcal{A}} = \{\bar{a}^{(1)}, \ldots, \bar{a}^{(m)}\}$ with $m$ levels. Therefore, evaluating such policies requires a mechanism to translate information from the historical action space $\mathcal{A}$ to the new action space $\bar{\mathcal{A}}$. This situation arises naturally when insurers wish to evaluate pricing adjustments that were not previously observed in the data, see Section~\ref{sec:extrapolation}.

Second, even when the action spaces coincide, it is desirable to exploit the structure of the action variable. In many pricing settings,  expected rewards vary smoothly with the price adjustment. Using this structure allows us to aggregate information across nearby actions and thereby reduce the variability of IPS estimates.

To address both issues simultaneously, we introduce kernel matrices  $\mathbf{K}(\bx) \in \mathbb{R}^{d \times m}$ that map between the two action spaces and smooth over nearby actions. Moreover, denote by $\bar{\be}_{\bar{a}} \in \{0,1\}^m$ the one-hot encoding vector corresponding to an action $\bar{a} \in \bar{\cal A}$.

We now exploit a deterministic policy $\bar{\pi}:{\cal X} \to \bar{\cal A}$ on this new action space $\bar{\cal A}$, and we set for the kernelized functional
\begin{equation*}
    \bar{\pi}_{\mathbf{K}}(a\mid\bx) = 
    \left\langle\be_{a}, \mathbf{K}(\bx)\, \bar{\be}_{\bar{\pi}(\bx)}\right\rangle
    =\left\langle \mathbf{K}(\bx)^\top
    \be_{a}, \,  \bar{\be}_{\bar{\pi}(\bx)}\right\rangle ~\in~ \R\,.
\end{equation*}
This can be interpreted as connecting the new actions $\bar{a} \in \bar{\cal A}$ to the old action space ${\cal A}$, e.g., action $\bar{a}^{(k)}$, $1\le k \le m$, provides a kernelized version on the original action space ${\cal A}$
\begin{equation*}
    \mathbf{K}(\bx) \be_{\bar{a}^{(k)}}
    = \left(\mathbf{K}_{1,k}(\bx), \ldots, 
    \mathbf{K}_{d,k}(\bx)
    \right)^\top~\in ~\R^d\,.
\end{equation*}
This structure is quite similar to the stochastic control \eqref{stochastic policy}. In particular, for our choice in Section~\ref{sec:kernel_construction}, the column sums of $\mathbf{K}(\bx)$ will aggregate to one, see Corollary~\ref{cor 1}, below. However, our choice of $\mathbf{K}(\bx)$ will not provide a probability tensor because the entries of $\mathbf{K}(\bx)$ will not necessarily be positive.

The {\it kernelized inverse propensity score} estimate can then be defined by
\begin{equation}\label{kernelized IPS}
    \widehat{V}_{\rm K}(\bar{\pi}) 
    = \frac{1}{n}\sum_{i=1}^n R_i\, \frac{\left\langle\be_{A_i},\, \mathbf{K}_i(\bX_i)\, \bar{\be}_{\bar{\pi}(\bX_i)}\right\rangle }{\widetilde{\pi}_i(A_i\mid\bX_i)}
    = \frac{1}{n}\sum_{i=1}^n R_i\, \frac{\bar{\pi}_{\mathbf{K}_i}(A_i\mid\bX_i)}{\widetilde{\pi}_i(A_i\mid\bX_i)}\,,
\end{equation}
where $\mathbf{K}_i$ may depend on the observed policy $\widetilde{\pi}_i$, $1\leq i \leq n$ used to generate observation $i$ in the learning sample ${\cal L}_{\widetilde{\bpi}}$. We will justify this choice and return to this formula in \eqref{kernelized IPS2} below.

\subsection{Kernel construction}\label{sec:kernel_construction}

Remark that the kernelized IPS estimate \eqref{kernelized IPS} is model-free. However, in the next step, we introduce a model to construct the kernel matrix $\mathbf{K}(\bx)$, for a given observed policy $\widetilde{\pi}$. Let us recall the regression function \eqref{regression direct model}, regressing the rewards $R$ from features $\bx \in \mathcal{X}$ and actions $a\in\mathcal{A}$. Similar to the direct method, we make the following regression assumption in the actions, for fixed $\bx$,
\begin{equation}\label{linear regression ansatz}
    a~\mapsto~\varrho(\bx, a) =  \beta_0(\bx) + \sum_{j=1}^q \beta_j(\bx) f_j(a) =   \bff(a)^\top\bbeta(\bx) \,,
\end{equation}
with regression parameter $\bbeta(\bx)= (\beta_0(\bx), \beta_1(\bx),\ldots,\beta_q(\bx))^\top \in \R^{q+1}$ and a vector of basis functions
$\bff(a) = (1, f_1(a),\ldots,f_q(a))^\top \in \R^{q+1}$ for
smooth functions $f_j : \mathbb{R} \to \mathbb{R}$, for \break $j=1,\ldots,q$. In contrast to the direct method, we do not specify a regression relation between the policyholder's features $\bx$ and the expected rewards $\varrho(\bx, a)$ in \eqref{linear regression ansatz}. Specifying a regression in the actions $a\in \mathcal{A}$ allows us to smooth over different actions and extrapolate to the new action space $\bar{\mathcal{A}}$. Thus, we implicitly extend the regression assumption \eqref{linear regression ansatz} to the actions $\bar{a} \in \bar{\mathcal{A}}$ on the new action space.\\
Choosing $q = 1$ and $f_1(a) = a$ simplifies to a linear regression in $a$. In this case, equation \eqref{linear regression ansatz} means that actions impact the expected rewards linearly, for given policyholders $\bx \in {\cal X}$. If this is not the case, one should use a different set of smooth functions. The regression assumption \eqref{linear regression ansatz} even accommodates a simple additive model on the action space $\mathcal{A}$ (with a linear link function).\\
Unlike in the direct method, we do not use the regression model \eqref{linear regression ansatz} to estimate the values $V(\pi)$ and $V(\bar{\pi})$, respectively, but we use it to determine the kernel matrix $\mathbf{K}(\bx)$. To this end, we need to specify the linear regression parameter $\bbeta(\bx) \in \R^{q+1}$ for any $\bx \in {\cal X}$.
Assume that for each action $a^{(k)} \in {\cal A}$, $1\le k \le d$, there is a reward $r_k(\bx) \in \R$, for given $\bx$. This gives the reward vector $\br(\bx)=(r_1(\bx),\ldots, r_d(\bx))^\top \in \R^d$ for given $\bx$.
This allows us to consider the weighted least squares problem to find the regression parameter $\widehat{\bbeta}(\bx)$,
\begin{eqnarray} \label{eq:weighted_least_squares}
    \widehat{\bbeta}(\bx)&\in&
    \underset{ \bbeta}{\arg\min}\,
    \left(\br(\bx) - D  \bbeta \right)^\top W(\bx) \left(\br(\bx) - D  \bbeta \right)\,,
\end{eqnarray}
with a symmetric positive-definite weight matrix $W \in \R^{d\times d}$ and design matrix $D$ given by
\begin{equation*}
D=\begin{pmatrix}1 & f_1(a^{(1)}) & \cdots & f_q(a^{(1)}) \\ \vdots & \vdots & \ddots & \vdots \\1 & f_1(a^{(d)}) & \cdots & f_q(a^{(d)}) \end{pmatrix}~\in~ \R^{d \times (q+1)}\,.
\end{equation*}
We assume that $D$ has full column rank $q+1$. Choosing
\begin{equation}\label{eq:diagonal_weight_matrix}
    W(\bx)={\rm diag}\left(\widetilde{\pi}(a^{(k)}\mid\bx)\right)_{k=1}^d\,,
\end{equation}
results in the standard weighted least squares formulation, minimizing the objective function
\begin{equation*}
    \sum_{k=1}^d \widetilde{\pi}(a^{(k)}\mid\bx)\left(r_k(\bx) - \bff(a^{(k)})^\top\, \bbeta \right)^2\,.
\end{equation*}
Thus, when the learning sample $\mathcal{L}_{\widetilde{\bpi}}$ contains observations generated under different observed policies $\widetilde{\pi}_i$, $1\leq i\leq n$, the corresponding weight matrices $W_i(\bx)$ may also differ across observations.
The above choice \eqref{eq:diagonal_weight_matrix} of the diagonal weight matrix $W(\bx)$ is motivated by classical regression settings with independent observations. Here, however, the reward vector $\br(\bx)$ collects rewards of a \emph{single} individual under different actions, which are generally correlated. For this reason, alternative specifications of the weight matrix $W(\bx)$ may be more appropriate. We return to this question in Propositions~\ref{prop:variance_bound}--\ref{prop:cond_moments}, below, and we will refer to the kernel matrix corresponding to the choice \eqref{eq:diagonal_weight_matrix} as the \textit{naive kernel matrix} leading to a \textit{naive kernelized IPS estimator}.

\medskip

The weighted least squares problem \eqref{eq:weighted_least_squares} has a unique solution since the matrix $D$ has full rank and $W(\bx)$ is symmetric positive-definite. The solution is given by
\begin{equation*}
    \widehat{\bbeta}(\bx) =\left(D ^\top W(\bx)  D\right)^{-1} 
    D ^\top W(\bx) \, \br(\bx)  = D^+_{\bx}\, \br(\bx)\,,
\end{equation*}
where we defined the matrix
\begin{equation*}
    D^+_{\bx} = \left(D ^\top W(\bx)  D\right)^{-1} 
    D ^\top W(\bx) ~\in~\R^{(q+1) \times d}\,.
\end{equation*}
Assuming that the regression assumption \eqref{linear regression ansatz} extends to $\bar{\cal A}$, we estimate the conditionally expected rewards for fixed $\bx$ as
\begin{equation}\label{estimated reward linear}
    \begin{pmatrix}
    \widehat{\varrho}(\bx, \bar{a}^{(1)})\\\vdots\\\widehat{\varrho}(\bx, \bar{a}^{(m)})
    \end{pmatrix}
     =
     \begin{pmatrix}1&f_1(\bar{a}^{(1)}) & \cdots & f_q(\bar{a}^{(1)}) \\ \vdots & \vdots & \ddots & \vdots \\1& f_1(\bar{a}^{(m)}) & \cdots & f_q(\bar{a}^{(m)}) \end{pmatrix}\widehat{\bbeta}(\bx) 
     =  \bar{D}\,
    D^+_{\bx}\, \br(\bx)\,,
\end{equation}
where we defined the design matrix on the new action space $\bar{\mathcal{A}}$,
\begin{equation*}
    \bar{D} = \begin{pmatrix}1&f_1(\bar{a}^{(1)}) & \cdots & f_q(\bar{a}^{(1)}) \\ \vdots & \vdots & \ddots & \vdots \\1& f_1(\bar{a}^{(m)}) & \cdots & f_q(\bar{a}^{(m)}) \end{pmatrix} ~\in ~\R^{m \times (q+1)}\,.
\end{equation*}
This directly links to the wanted kernel matrix. 
Namely, we set 
\begin{align}
    \mathbf{K}(\bx) = (D^+_{\bx})^\top \bar{D}^\top = W(\bx) D \left(D ^\top W(\bx)  D\right)^{-1} \bar{D} ^\top ~\in ~\R^{d \times m}\,. \label{eq:kernel_matrix}
\end{align}
Taking an action $\bar{\pi}(\bx) \in \bar{\mathcal{A}}$, gives an estimated expected reward
$\widehat{\varrho}(\bx, \bar{\pi}(\bx))$. This estimated expected reward
is given by, see \eqref{estimated reward linear},
\begin{equation*}
    \widehat{\varrho}(\bx, \bar{\pi}(\bx))=
    \bar{\be}_{\bar{\pi}(\bx)}^\top 
     \mathbf{K}(\bx)^\top \br(\bx)
     =\left(\mathbf{K}(\bx)\bar{\be}_{\bar{\pi}(\bx)}\right)^\top
    \br(\bx)\,.
\end{equation*}
The final step is to consider action $A_i$ with corresponding reward $R_i$, for given $\bX_i$, $1\leq i\leq n$, from the learning sample ${\cal L}_{\widetilde{\bpi}}$. This leads to the {\it kernelized inverse propensity score} estimate
\begin{equation}\label{kernelized IPS2}
    \widehat{V}_{\rm K}(\bar{\pi}) 
    = \frac{1}{n}\sum_{i=1}^n R_i\, \frac{\left\langle\be_{A_i},\, \mathbf{K}_i(\bX_i)\, \bar{\be}_{\bar{\pi}(\bX_i)}\right\rangle }{\widetilde{\pi}_i(A_i\mid\bX_i)}
    = \frac{1}{n}\sum_{i=1}^n R_i\, \frac{\left\langle \mathbf{K}_i(\bX_i)^\top
    \be_{A_i}, \,  \bar{\be}_{\bar{\pi}(\bX_i)}\right\rangle }{\widetilde{\pi}_i(A_i\mid\bX_i)}\,,
\end{equation}
which is precisely as defined in \eqref{kernelized IPS}, and with the kernel matrices obtained from \eqref{eq:kernel_matrix}.

The following example illustrates the difference between the IPS estimator and its kernelized counterpart.

\begin{example}\label{ex:small_example} \normalfont
    In this example, we study the kernelized IPS estimator under the assumption that $\mathcal{A} = \bar{\mathcal{A}}$, i.e., we only want to study the effect of kernelization on the same action space.\\
    Assume that we observe three policyholders and, for simplicity, let them have the same features $\bx_1 = \bx_2 = \bx_3 \in \mathcal{X}$. Based on a default price of $100$, we offer each policyholder \mbox{$i\in \{1, 2, 3\}$} a premium reduction $A_i\in\mathcal{A} = \{10\%, 20\%, 30\%\}$ with uniform probability \mbox{$\widetilde{\pi}_i(a\mid\bx)=\frac{1}{3}$} for all $a\in \mathcal{A}$. Let the reward be defined as the total signed premium. Thus, if the policyholder does not accept the offer, the reward is 0, and $(1-A_i) \,100$ otherwise. We observe the following pairs of actions and rewards, which form the learning sample $\mathcal{L}_{\widetilde{\bpi}}$.
    \begin{center}
        \begin{tabular}{c|c|c}
             Policyholder $i$ & Action $A_i$ & Reward $R_i$ \\ \hline
             1 & 10\% & 90 \\
             2 & 20\% & 0\\
             3 & 30\% & 70
        \end{tabular}
    \end{center}
    We use this information to approximate the value of three new constant deterministic policies $\bar{\pi}^{(k)}(\bx) = a^{(k)} = 10k\,\%$, $1\leq k \leq 3$, for $a^{(k)} \in \bar{\mathcal{A}} = \mathcal{A}$ and for all $\bx \in \mathcal{X}$.\\
    Using the IPS estimator, we calculate
    \begin{align*}
        \widehat{V}_{\rm IPS}(\bar{\pi}^{(1)}) = \frac{1}{n}\sum_{i=1}^n R_i\, \frac{\mathds{1}_{\{A_i=\bar{\pi}^{(1)}(\bX_i)\}}}{\widetilde{\pi}_i(A_i\mid\bX_i)} = \frac{1}{3} \left( 90 \, \frac{1}{1/3} + 0 \,\frac{0}{1/3} + 70\,\frac{0}{1/3} \right) = 90\,,
    \end{align*}
    and equivalently $\widehat{V}_{\rm IPS}(\bar{\pi}^{(2)}) = 0$ and $\widehat{V}_{\rm IPS}(\bar{\pi}^{(3)}) = 70$.

    For the kernelized IPS estimator, we make a locally linear regression assumption 
    \begin{equation*}\label{eq:local_linear_regression}
        \varrho(\bx, a) = \beta_0(\bx) + \beta_1(\bx) \, a\,.
    \end{equation*}
    We choose the weight matrix $W(\bx) = W_i(\bx) ={\rm diag}\left(\widetilde{\pi}_i(a^{(k)}|\bx)\right)_{k=1}^d $, $1\leq i \leq 3$ and the design matrix $D$ follows as
    \begin{equation*}
        W(\bx) = \begin{pmatrix}
            \frac{1}{3} & 0 & 0\\
            0 & \frac{1}{3} & 0\\
            0 & 0 & \frac{1}{3}
        \end{pmatrix}\qquad  \text{and} \qquad D = \begin{pmatrix}
            1 & 0.1\\
            1 & 0.2\\
            1 & 0.3
        \end{pmatrix}\,.
    \end{equation*}
    Analogously to the weight matrix $W(\bx)$, the kernel matrix $\mathbf{K}(\bx)$ is identical for all policyholders because the observed policies $\widetilde{\pi}_i$, $1\leq i\leq 3$, are identical. Hence,
    \begin{equation*}
        \mathbf{K}(\bx) = W(\bx) D \left(D ^\top W(\bx)  D\right)^{-1} D ^\top
        \approx \begin{pmatrix}
            0.833 & 0.333 & -0.167\\
            0.333 & 0.333 & 0.333\\
            -0.167 & 0.333 & 0.833
        \end{pmatrix}\,.
    \end{equation*}
    Note that the columns of $\mathbf{K}(\bx)$ aggregate to one, but it is not a probability tensor. We use this to determine the kernelized IPS estimator
    \begin{align*}
        \widehat{V}_{\rm K}(\bar{\pi}^{(1)}) = \frac{1}{n}\sum_{i=1}^n R_i\, \frac{\left\langle\be_{A_i},\, \mathbf{K}(\bX_i)\, \bar{\be}_{\bar{\pi}^{(1)}(\bX_i)}\right\rangle }{\widetilde{\pi}_i(A_i\mid\bX_i)} = \frac{1}{3} \left( 90 \, \frac{0.833}{1/3} + 0 \,\frac{0.333}{1/3} - 70\,\frac{0.167}{1/3} \right) = 63.28\,,
    \end{align*}
    and equivalently $\widehat{V}_{\rm K}(\bar{\pi}^{(2)}) = 53.28$ and $\widehat{V}_{\rm K}(\bar{\pi}^{(3)}) = 43.28$\,.

    \begin{figure}[htbp!]
    	\centering
        \includegraphics[width=0.6\textwidth]{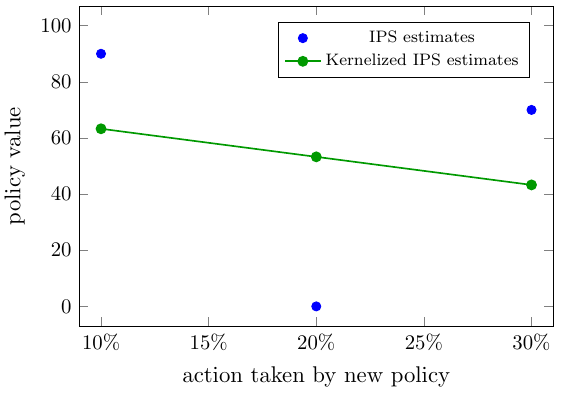}
        \caption{Kernelized and IPS estimates of policy value for constant policies under the setting in Example~\ref{ex:small_example}.}
        \label{fig:example_interpolation}
    \end{figure}

    The IPS estimator only relies on the information provided by observations where recorded actions match the action of the new policy. The kernelized IPS estimator, on the other hand, takes into account the rewards generated under different actions and smooths the rewards across the action space $\mathcal{A}$. The kernelized IPS estimates lie on the regression line obtained by fitting the locally linear model \eqref{eq:local_linear_regression} to the IPS estimates. This is shown in Figure~\ref{fig:example_interpolation}. From this illustration, it is intuitively clear that the kernelized IPS estimator can reduce the variance in estimation compared to the IPS estimator. In the next subsection, we formalize this observation.\\
    This completes the example.
\end{example}

\subsection{Statistical properties}

In this subsection, we analyze the statistical properties of the kernelized IPS estimator. In particular, we prove its unbiasedness under the local regression assumption~\eqref{linear regression ansatz} and compare its variance to that of the classical IPS estimator. The regression structure introduced in Section~\ref{sec:kernel_construction} is used to construct a variance-optimal kernel matrix, and it plays a key role in establishing these results.

The kernelized IPS estimator \eqref{kernelized IPS2} itself is model-free. However, we have assumed a regression structure \eqref{linear regression ansatz} in the actions $a \in \mathcal{A}$ to determine an explicit form of the kernel matrix $\mathbf{K}(\bx)$. This regression structure smooths over the action space and is used again to prove unbiasedness in the following Proposition.

\begin{proposition}\label{proposition 2}
Under the regression assumption \eqref{linear regression ansatz}, the kernelized IPS estimator \eqref{kernelized IPS2} is unbiased for the value $V(\bar{\pi})$.
\end{proposition}
{\Beweis
{\bf Proof.} In fact, we will show that every summand of the kernelized IPS estimator
\begin{equation*}
    \widehat{V}_{\rm K}(\bar{\pi}) 
    = \frac{1}{n}\sum_{i=1}^n R_i\, \frac{\left\langle\be_{A_i},\, \mathbf{K}_i(\bX_i)\, \bar{\be}_{\bar{\pi}(\bX_i)}\right\rangle }{\widetilde{\pi}_i(A_i\mid\bX_i)}\,,
\end{equation*}
is an unbiased estimator for $\E_{\bar{\pi}}[R_i]$. Consider a single tuple $(\bX, A, R)$, where the action $A$ follows some policy $\widetilde{\pi}$, leading to a kernel matrix $\mathbf{K}(\bX)$. We compute the expected value under the regression assumption \eqref{linear regression ansatz}
\begin{eqnarray*}
\E_{\widetilde{\pi}}\left[ R \frac{\left\langle\be_{A},\, \mathbf{K}(\bX)\, \bar{\be}_{\bar{\pi}(\bX)}\right\rangle }{\widetilde{\pi}(A\mid\bX)}
\right]
&=&\E_{\widetilde{\pi}}\left[ \E\left[ R \mid\bX, A \right] \frac{\left\langle\be_{A},\, \mathbf{K}(\bX)\, \bar{\be}_{\bar{\pi}(\bX)}\right\rangle }{\widetilde{\pi}(A\mid\bX)}
\right]
= \E_{\widetilde{\pi}}\left[ \bff(A)^\top\, \bbeta(\bX)\, \frac{\left\langle\be_{A},\, \mathbf{K}(\bX)\, \bar{\be}_{\bar{\pi}(\bX)}\right\rangle }{\widetilde{\pi}(A\mid\bX)}
\right]
\\&=& \E\left[\, \sum_{ a \in {\cal A}}
\bff(a)^\top\, \bbeta(\bX) \left\langle\be_{a},\, \mathbf{K}(\bX)\, \bar{\be}_{\bar{\pi}(\bX)}\right\rangle\right]
\\
&=& \E\left[\left( \sum_{ a \in {\cal A}}\bff(a)^\top\, \bbeta(\bX) \be_{a}^\top\right) \mathbf{K}(\bX)\, \bar{\be}_{\bar{\pi}(\bX)}\right]
\\&=& \E\left[\left(D \,\bbeta(\bX)\right)^\top \mathbf{K}(\bX)\, \bar{\be}_{\bar{\pi}(\bX)}\right]
= \E\left[\bbeta(\bX)^\top D^\top  \mathbf{K}(\bX)\, \bar{\be}_{\bar{\pi}(\bX)}\right]
.
\end{eqnarray*}
The crucial step now is the following simplification: we use the fact that $D^\top W D$ is symmetric,
\begin{eqnarray*}
D^\top  \mathbf{K}(\bX)
&=& D^\top(D^+_{\bX})^\top \bar{D}^\top
~=~ D^\top\left( \left(D ^\top W(\bX)  D\right)^{-1} 
D ^\top W(\bX)\right)^\top \bar{D}^\top
\\&=& D^\top W(\bX) D \left(D ^\top W(\bX)  D\right)^{-1}  \bar{D}^\top
~=~\bar{D}^\top\,.
\end{eqnarray*}
Inserting this into the previous expected value and reverting all the previous steps allows us to obtain (using the regression assumption once more)
\begin{eqnarray*}
\E_{\widetilde{\pi}}\left[ R \frac{\left\langle\be_{A},\, \mathbf{K}(\bX)\, \bar{\be}_{\bar{\pi}(\bX)}\right\rangle }{\widetilde{\pi}(A\mid\bX)}
\right] 
&=& \E\left[\bbeta(\bX)^\top D^\top  \mathbf{K}(\bX)\, \bar{\be}_{\bar{\pi}(\bX)}\right] = \E\left[\bbeta(\bX)^\top \bar{D}^\top\, \bar{\be}_{\bar{\pi}(\bX)}\right] \\
&=&\E\left[\left(\sum_{ \bar{a} \in \bar{\cal A}}
\bff(\bar{a})^\top\, \bbeta(\bX) \bar{\be}_{\bar{a}}^\top\right) \bar{\be}_{\bar{\pi}(\bX)}\right]
=\E\left[\,\sum_{ \bar{a} \in \bar{\cal A}}
\E \left[ R \mid\bX, A = \bar{a} \right]  \mathds{1}_{\{\bar{\pi}(\bX) = \bar{a}\}}\right]\\ [0.1cm]
&=&\E\left[\E_{\bar{\pi}} \left[R \mid \bX\right] \right]
~=~\E_{\bar{\pi}} \left[ R  \right]\,.
\end{eqnarray*}
This completes the proof.
\EndProof}

\medskip

In the previous proof, we have verified the following corollary.

\begin{cor}\label{cor 1}
We have for all $\bx \in {\cal X}$
\begin{equation*}
 \mathbf{K}(\bx)^\top D =\bar{D}\,.
\end{equation*}
The column sums of $ \mathbf{K}(\bx)$ are equal to one, since the regression assumption \eqref{linear regression ansatz} contains a bias term $\beta_0(\bx)$.
\end{cor}
This corollary shows how the actions in ${\cal A}$ are mapped to $\bar{\cal A}$ by the kernel $\mathbf{K}(\bx)$. We can also consider the special case of $\bar{\cal A}={\cal A}$. In this case, we have $\bar{D}=D$, which provides the kernel matrix
\begin{equation*}
\mathbf{K}(\bx) = (D^+_{\bx})^\top \bar{D}^\top
=W(\bx)
D \left(D ^\top W(\bx)  D\right)^{-1} D ^\top\,.
\end{equation*}
Generally, this is different from the identity matrix, but
Corollary \ref{cor 1} immediately gives
\begin{equation*}
\mathbf{K}(\bx)^\top D =D\,,
\end{equation*}
by multiplying the kernel matrix $\mathbf{K}(\bx)$ from the left with $D^\top$.

\medskip

Under $\bar{\cal A}={\cal A}$, we can also multiply from the right, providing
\begin{equation*}
\mathbf{K}(\bx) W(\bx)
D=W(\bx)
D \left(D ^\top W(\bx)  D\right)^{-1} D ^\top W(\bx)
D=W(\bx)
D\,.
\end{equation*}

In the special case $\bar{\cal A}={\cal A}$, the identities $\mathbf{K}(\bx)^\top D = D$ and $\mathbf{K}(\bx)W(x)D = W(\bx)D$ show that $\mathbf{K}(\bx)$ leaves the regression subspace spanned by $D$ invariant and thus acts like a (weighted) projection onto the subspace given by the local model $\eqref{linear regression ansatz}$. Intuitively, applying this projection to the IPS weights preserves the components needed for unbiasedness while discarding the components orthogonal to that subspace. This is precisely where the variance can be reduced.

\medskip
Similarly, Example~\ref{ex:small_example} suggests that the kernelized IPS estimator is less noisy than the IPS estimator since it uses information from all observations instead of only perfect matches. The following result introduces an upper bound for the variance of the kernelized IPS estimator relative to the IPS estimator. This upper bound is a function of the chosen weight matrix $W(\bx)$. We will use this result to determine an optimal weight matrix in the following.

\begin{proposition}[Variance bound]\label{prop:variance_bound}
For  $1\le i \le n$, let the weight matrix $W_i(\bX_i) \in \R^{d\times d}$ be any symmetric positive-definite matrix and choose the corresponding kernel matrix $\mathbf{K}_i(\bX_i)$ as in \eqref{eq:kernel_matrix}.

Denote by $\Sigma_i(\bX_i)$ the conditional covariance matrix of the per-action IPS weights with entries 
\begin{equation*}
    \Sigma_{i,jk}(\bX_i) = \operatorname{Cov}_{\widetilde{\pi}_i}\left(R_i \,\frac{\mathds{1}_{\{A_i=a^{(j)}\}}}{\widetilde\pi_i(a^{(j)}\mid \bX_i)}, \, R_i \,\frac{\mathds{1}_{\{A_i=a^{(k)}\}}}{\widetilde\pi_i(a^{(k)}\mid \bX_i)} \;\middle|\; \bX_i\right)\,.
\end{equation*}
Then, for any policy $\pi$ defined on the action space $\bar{\mathcal{A}} = \mathcal{A}$, we have
\begin{equation}\label{eq:final-mismatch}
    \Var_{\widetilde{\bpi}}\big(\widehat{V}_{\rm K}(\pi)\mid \bX_1, \dots, \bX_n\big)
    \;\le\;
    \max_{1\leq i \leq n} \left(\kappa\big(W_i(\bX_i)\,\Sigma_i(\bX_i)\big) \right)\;
    \Var_{\widetilde{\bpi}}\big(\widehat{V}_{\mathrm{IPS}}(\pi)\mid \bX_1, \dots, \bX_n\big)\,,
\end{equation}
where for a square matrix $M$, $\kappa(M) \geq 1$ denotes the spectral condition number of $M$, defined as the ratio between the maximum and the minimum eigenvalue of $M$.
\end{proposition}

{\Beweis
{\bf Proof.}
Recall that both estimators can be written as averages of $n$ independent summands. Conditionally on the features $\bX_1,\dots,\bX_n$, the tuples $(\bX_i,A_i,R_i)$ are independent, and therefore the summands in $\widehat{V}_{\mathrm{IPS}}(\pi)$ and $\widehat{V}_{\rm K}(\pi)$ are independent as well. Hence, we obtain
\begin{align*}
    \Var_{\widetilde{\bpi}}\big(\widehat{V}_{\mathrm{IPS}}(\pi)\mid \bX_1,\dots,\bX_n\big)
    &= \frac{1}{n^2}\sum_{i=1}^n 
    \Var_{\widetilde{\pi}_i}\!\left(
        R_i\frac{\mathds{1}_{\{A_i=\pi(\bX_i)\}}}{\widetilde{\pi}_i(A_i\mid\bX_i)}
        \,\middle|\, \bX_i
    \right),\\
    \Var_{\widetilde{\bpi}}\big(\widehat{V}_{\rm K}(\pi)\mid \bX_1,\dots,\bX_n\big)
    &= \frac{1}{n^2}\sum_{i=1}^n 
    \Var_{\widetilde{\pi}_i}\!\left(
        R_i\frac{\langle \be_{A_i},\mathbf K_i(\bX_i)\be_{\pi(\bX_i)}\rangle}{\widetilde{\pi}_i(A_i\mid\bX_i)}
        \,\middle|\, \bX_i
    \right).
\end{align*}

Thus, we can treat every instance separately. Therefore, similarly to the previous proof, we consider a single tuple $(\bX,A,R)$ where the action $A$ is determined by a policy $\widetilde{\pi}$, which leads to a kernel matrix $\mathbf{K}(\bX)$.

Let $\bY\in\mathbb{R}^d$ denote the vector of per-action IPS weights with entries
\begin{equation*}
    Y_j ~=~ R\,\frac{\mathds{1}_{\{A=a^{(j)}\}}}{\widetilde{\pi}(a^{(j)}\mid\bX)}\,.
\end{equation*}

Then, we can write
\begin{equation*}
    \widehat{V}_{\mathrm{IPS}}(\pi) ~=~ \be_{\pi(\bX)}^\top \bY \qquad \text{and}
    \qquad
    \widehat{V}_{\rm K}(\pi) ~=~ \be_{\pi(\bX)}^\top(\mathbf K(\bX)^\top\bY)\,.
\end{equation*}

We have defined $\Sigma(\bX)$ as the conditional covariance matrix of $\bY$. Then
\begin{equation*}
    \operatorname{Var}_{\widetilde{\pi}}\big(\widehat{V}_K(\pi)\mid \bX\big)
    = \operatorname{Var}_{\widetilde{\pi}}\big(\be_{\pi(\bX)}^\top \mathbf{K}(\bX)^\top \bY\mid \bX\big) = \be_{\pi(\bX)}^\top \mathbf{K}(\bX)^\top\,\Sigma(\bX)\,\mathbf{K}(\bX) \be_{\pi(\bX)}\,,
\end{equation*}
and analogously
\begin{equation*}
\operatorname{Var}_{\widetilde{\pi}}\big(\widehat{V}_{\mathrm{IPS}}(\pi)\mid \bX\big)
= \be_{\pi(\bX)}^\top \Sigma(\bX) \be_{\pi(\bX)}\,.
\end{equation*}

Now, set $Z=\Sigma(\bX)^{1/2}\be_{\pi(\bX)}$. Then
\begin{align*}
    \Var_{\widetilde{\pi}}(\widehat{V}_{\rm K}\mid \bX)
    &= Z^\top\big(\Sigma(\bX)^{-1/2}\mathbf{K}(\bX)^\top\Sigma(\bX) \mathbf{K}(\bX)\Sigma(\bX)^{-1/2}\big)Z\\
    &= \big\lVert(\Sigma(\bX)^{-1/2}\mathbf{K}(\bX)^\top\Sigma(\bX)^{1/2})\,Z\big\rVert_2^2
    \le \lVert\Sigma(\bX)^{-1/2}\mathbf{K}(\bX)^\top\Sigma(\bX)^{1/2}\rVert_2^2\,\lVert Z\rVert_2^2,
\end{align*}
and $\lVert Z\rVert_2^2=\Var_{\widetilde{\pi}}(\widehat{V}_{\mathrm{IPS}}\mid \bX)$.

Factor $\mathbf{K}(\bX)^\top$ as $\mathbf{K}(\bX)^\top=W(\bX)^{-1/2}P(\bX)W(\bX)^{1/2}$ with
\begin{equation*}
P(\bX)=W(\bX)^{1/2}D\big(D^\top W(\bX) D\big)^{-1}D^\top W(\bX)^{1/2},
\end{equation*}
which is an ordinary Euclidean projector and hence $\lVert P(\bX)\rVert_2=1$ \citep[see, e.g.,][Chapter 12]{Harville1997}.\,
Then
\begin{align*}
    \lVert\Sigma(\bX)^{-1/2}\mathbf{K}(\bX)^\top\Sigma(\bX)^{1/2}\rVert_2^2
    &= \bigl\lVert(\Sigma(\bX)^{-1/2}W(\bX)^{-1/2})\,P(\bX)\,(W(\bX)^{1/2}\Sigma(\bX)^{1/2})\bigr\rVert_2^2\\
    & \le \lVert\Sigma(\bX)^{-1/2}W(\bX)^{-1/2}\rVert_2^2 \,\lVert W(\bX)^{1/2}\Sigma(\bX)^{1/2}\rVert_2^2.
\end{align*}
Since
$\lVert W(\bX)^{-1/2}\Sigma(\bX)^{-1/2}\rVert_2^2 = \lambda_{\max}(\Sigma(\bX)^{-1} W(\bX)^{-1}) = 1/\lambda_{\min}(W(\bX)\Sigma(\bX))$ is the inverse of the minimum eigenvalue of $W(\bX)\Sigma(\bX)$ and $\lVert\Sigma(\bX)^{1/2}W(\bX)^{1/2}\rVert_2^2=\lambda_{\max}(W(\bX)\Sigma(\bX))$ its maximum eigenvalue, their product equals $\kappa(W(\bX)\Sigma(\bX))$. Now, aggregating over several tuples $(\bX_i, A_i, R_i)$, $1\leq i \leq n$, yields \eqref{eq:final-mismatch}.
\EndProof}\\

Note that if we choose $W_i(\bX_i)=\Sigma_i(\bX_i)^{-1}$, $1\leq i \leq n$, we have $\kappa(W_i(\bX_i)\Sigma_i(\bX_i))=1$, which motivates the following corollary.

\begin{cor}\label{cor:lowest_bound}
    Let the weight matrices $W_i(\bX_i) = \Sigma_i(\bX_i)^{-1}$, $1\leq i \leq n$, equal the inverse of the conditional covariance matrix as defined in Proposition~\ref{prop:variance_bound}. Under the regression assumption \eqref{linear regression ansatz}, we have
    \begin{equation*}
        \operatorname{Var}_{\widetilde{\bpi}}\big(\widehat{V}_{\rm K}(\pi)\big)  ~ \le~ \operatorname{Var}_{\widetilde{\bpi}}\big(\widehat{V}_{\mathrm{IPS}}(\pi)\big)\,,
    \end{equation*}
    for any policy $\pi$ defined on the action space $\bar{\mathcal{A}} = \mathcal{A}$.
\end{cor}
{\Beweis
{\bf Proof.} 
By the law of total variance,
\begin{equation*}
\operatorname{Var}_{\widetilde{\bpi}}\big(\widehat{V}_{M}(\pi)\big)
= \E\left[\operatorname{Var}_{\widetilde{\bpi}}\big(\widehat{V}_{M}(\pi)\mid \bX_1, \dots, \bX_n\big)\right]
+ \operatorname{Var}\left(\E_{\widetilde{\bpi}}\big[\widehat{V}_{M}(\pi)\mid \bX_1, \dots, \bX_n\big]\right),
\end{equation*}
for $M \in \{ K, \text{IPS} \}$.
The second term is identical for $\widehat{V}_K$ and $\widehat{V}_{\mathrm{IPS}}$ since both estimators are conditionally unbiased, see the proofs of Propositions~\ref{prop:unbiased_IPS} and \ref{proposition 2}. Now the result follows immediately with Proposition~\ref{prop:variance_bound}.
\EndProof}\\

We have just demonstrated that choosing the weight matrices $W_i(\bX_i) = \Sigma_i(\bX_i)^{-1}$, $1\leq i\leq n$, ensures that the kernelized IPS estimator reduces the variance relative to the IPS estimator. Additionally, it yields the lowest upper bound in \eqref{eq:final-mismatch} among all choices for the weight matrices $W_i(\bX_i)$. In fact, the following proposition shows that this choice minimizes the variance among all symmetric positive-definite weight matrices $W_i(\bX_i) \in \R^{d\times d}$.

\begin{proposition}[Minimal variance]\label{prop:min_variance}
    Let the conditional covariance matrix $\Sigma_i(\bX_i)$, $1\leq i\leq n$, be defined as in Proposition~\ref{prop:variance_bound}.
    Let 
    \begin{equation*}
        \mathbf{K}_i^\ast(\bX_i) ~=~ \Sigma_i(\bX_i)^{-1} D\big(D^\top \Sigma_i(\bX_i)^{-1} D\big)^{-1}\bar{D}^\top \,.
    \end{equation*}
    If the regression assumption \eqref{linear regression ansatz} holds, then we have, for any policy $\bar{\pi}$ defined on the action space $\bar{\mathcal{A}}$,
    \begin{equation*}
        \operatorname{Var}_{\widetilde{\bpi}}\big(\widehat{V}_{\mathrm{K}^\ast}(\bar{\pi})\big) ~\le~  \operatorname{Var}_{\widetilde{\bpi}}\big(\widehat{V}_{\widetilde{\mathrm{K}}}(\bar{\pi})\big)\,,
    \end{equation*}
    for any symmetric positive-definite weight matrix $\widetilde{W}_i(\bX_i)\in \R^{d\times d}$ leading to kernel matrices
    \begin{equation*}
        \widetilde{\mathbf{K}}_i(\bX_i) ~=~ \widetilde{W}_i(\bX_i) D\big(D^\top \widetilde{W}_i(\bX_i) D\big)^{-1}\bar{D}^\top\,.
    \end{equation*}
\end{proposition}

{\Beweis
{\bf Proof.}
As in the proof of Proposition~\ref{prop:variance_bound}, we introduce a single tuple $(\bX, A, R)$, with some policy $\widetilde{\pi}$ determining the distribution of the action $A$ and leading to a kernel matrix $\mathbf{K}(\bX)$.
As argued in Corollary~\ref{cor:lowest_bound}, it is sufficient to minimize the conditional variance of the kernelized IPS estimator, given $\bX$. Recall that
\begin{equation*}
    \operatorname{Var}_{\widetilde{\pi}}\big(\widehat{V}_K(\bar{\pi})\mid \bX\big)
    = \be_{\bar{\pi}(\bX)}^\top \mathbf{K}(\bX)^\top\,\Sigma(\bX)\,\mathbf{K}(\bX) \be_{\bar{\pi}(\bX)} \,,
\end{equation*}
and set $\omega = \mathbf{K}(\bX)\be_{\bar{\pi}(\bX)}$. Then Corollary~\ref{cor 1} yields 
\begin{equation*}
    D^\top \omega = \bar{D}^\top \be_{\bar{\pi}(\bX)}\,.
\end{equation*}
Thus, we determine the weight $\omega \in \R^d$ which minimizes the conditional variance under this constraint via the method of Lagrange multipliers
\begin{equation}
    \min_{\omega \in \R^d} L(\omega, \lambda) = \min_{\omega \in \R^d} \omega^\top \Sigma(\bX) \omega - \lambda^\top(D^\top \omega - \bar{D}^\top \be_{\bar{\pi}(\bX)})\,,
\end{equation}
for some $\lambda \in \R^{q+1}$. The following derivation is standard and included for completeness.
The first order condition in $\omega$ gives
\begin{equation*}
    0 = \frac{\partial L}{\partial \omega} = \Sigma(\bX) \omega - D\lambda \implies \omega = \Sigma(\bX)^{-1} D \lambda.
\end{equation*}
Similarly, considering $\lambda$, yields the constraint
\begin{equation*}
    \bar{D}^\top \be_{\bar{\pi}(\bX)} = D^\top \omega = D^\top \Sigma(\bX)^{-1} D \lambda\,.
\end{equation*}
We solve for $\lambda$ and find
\begin{equation*}
    \lambda = (D^\top \Sigma(\bX)^{-1} D)^{-1} \bar{D}^\top \be_{\bar{\pi}(\bX)}\,,
\end{equation*}
and therefore
\begin{equation*}
    \omega = \Sigma(\bX)^{-1} D \lambda = \Sigma(\bX)^{-1} D (D^\top \Sigma(\bX)^{-1} D)^{-1} \bar{D}^\top \be_{\bar{\pi}(\bX)} = \mathbf{K}^\ast \be_{\bar{\pi}(\bX)}\,. 
\end{equation*}
This is precisely our kernelized IPS estimator with weight matrix $W(\bX) = \Sigma^{-1}(\bX)$.
\EndProof}\\

We will refer to the kernelized IPS estimator corresponding to the choice $W_i(\bX_i)=\Sigma_i(\bX_i)^{-1}$, $1\leq i\leq n$, as the \textit{variance-optimal kernelized IPS estimator}. This is in contrast to the naive approach \eqref{eq:diagonal_weight_matrix}.
To use the previous result in practice, we require an explicit description of the conditional covariance matrix, given the policyholder features $\bX$. The next proposition provides these conditional moments.

\begin{proposition}[Entries of the covariance matrix]\label{prop:cond_moments}
    Consider a single tuple of policyholder features, action, and reward $(\bX, A, R)$. Assume that the action $A$ follows a policy $\widetilde{\pi}$.
    For $\bx \in \mathcal{X}$, define
    \begin{equation*}
        \mu_j(\bX)~=~\E[R\mid \bX,A=a^{(j)}]\,,\qquad 
        \sigma_j^2(\bX)~=~\operatorname{Var}(R\mid \bX,A=a^{(j)})\,.
    \end{equation*}
    The conditional covariance matrix $\Sigma(\bX)$, as defined in Proposition~\ref{prop:variance_bound}, is given by
    \begin{equation*}
        \Sigma_{jj}(\bX) \,=\, \frac{\sigma_j^2(\bX)}{\widetilde{\pi}(a^{(j)}\mid\bX)} \,+\, \mu_j(\bX)^2\,\frac{1
        -\widetilde{\pi}(a^{(j)}\mid\bX)}{\widetilde{\pi}(a^{(j)}\mid\bX)}\,,
    \end{equation*}
    and, for $j\neq k$,
    \begin{equation*}
        \Sigma_{jk}(\bX)\,=\, -\,\mu_j(\bX)\,\mu_k(\bX)\,.
    \end{equation*}
\end{proposition}

{\Beweis
{\bf Proof.} 
Let $\bY\in\mathbb{R}^d$, denote the vector of per-action IPS weights with entries
\begin{equation*}
    Y_{j} ~=~ R \,\frac{\mathds{1}_{\{A=a^{(j)}\}}}{\widetilde\pi(a^{(j)}\mid \bX)}\,.
\end{equation*}

First, for each $j \in \{1, \dots, d\}$, we find
\begin{equation*}
    \E_{\widetilde{\pi}}[Y_j \mid \bX]
    ~=~ \E_{\widetilde{\pi}}\left[R\frac{\mathds{1}_{\{A=a^{(j)}\}}}{\widetilde{\pi}(a^{(j)}\mid\bX)} \;\middle|\; \bX\right]
    = \frac{\widetilde{\pi}(a^{(j)}\mid\bX)}{\widetilde{\pi}(a^{(j)}\mid\bX)}\,\E[R \mid \bX, A=a^{(j)}]
    = \mu_j(\bX).
\end{equation*}
Next, we determine
\begin{eqnarray*}
    \E_{\widetilde{\pi}}[Y_j^2\mid \bX]
    &=& \E_{\widetilde{\pi}}\left[R^2\frac{\mathds{1}_{\{A=a^{(j)}\}}}{\widetilde{\pi}(a^{(j)}\mid\bX)^2}\;\middle|\; \bX\right]
    = \frac{\widetilde{\pi}(a^{(j)}\mid\bX)}{\widetilde{\pi}(a^{(j)}\mid\bX)^2}\,\E[R^2\mid \bX, A=a^{(j)}]\\
    &=& \frac{\E[R^2\mid \bX,A=a^{(j)}]}{\widetilde{\pi}(a^{(j)}\mid\bX)}
    = \frac{\sigma_j(\bX)^2 + \mu_j(\bX)^2}{\widetilde{\pi}(a^{(j)}\mid\bX)}\,,
\end{eqnarray*}
hence
\begin{equation*}
    \Sigma_{jj}(\bX)=\operatorname{Var}_{\widetilde{\pi}}(Y_j\mid \bX)=\E_{\widetilde{\pi}}[Y_j^2\mid \bX]-\big(\E_{\widetilde{\pi}}[Y_j\mid \bX]\big)^2
    = \frac{\sigma_j(\bX)^2 + \mu_j(\bX)^2}{\widetilde{\pi}(a^{(j)}\mid\bX)} - \mu_j(\bX)^2\,.
\end{equation*}
For $j\neq k$, $1\leq j, k \leq d$ since at most one action is taken,
$\mathds{1}_{\{A=a^{(j)}\}}\mathds{1}_{\{A=a^{(k)}\}}=0$ almost surely, so
\begin{equation*}
    \E_{\widetilde{\pi}}[Y_j Y_k\mid \bX]
    = \E_{\widetilde{\pi}}\left[R^2\,\frac{\mathds{1}_{\{A=a^{(j)}\}}\mathds{1}_{\{A=a^{(k)}\}}}{\widetilde{\pi}(a^{(j)}\mid\bX)\, \widetilde{\pi}(a^{(k)}\mid\bX)}\;\middle|\; \bX\right]
    = 0\,,
\end{equation*}
and therefore
\begin{equation*}
    \Sigma_{jk}=\operatorname{Cov}_{\widetilde{\pi}}(Y_j,Y_k\mid \bX)=\E_{\widetilde{\pi}}[Y_j Y_k\mid \bX]-\E_{\widetilde{\pi}}[Y_j\mid \bX]\E_{\widetilde{\pi}}[Y_k\mid \bX]
    = -\,\mu_j(\bX)\,\mu_k(\bX)\,.
\end{equation*}
\EndProof}\\

The previous results imply that achieving the variance-optimal kernelized IPS estimator requires certain conditional moments of the reward, which in practice must be estimated. One option is to approximate these moments using a regression model, similar to the direct method \eqref{regression direct model}. Importantly, however, the kernelized IPS estimator does not rely on the global model in the same way as the direct method. Its unbiasedness relies only on the local regression assumption \eqref{linear regression ansatz} used to construct the kernel matrix $\mathbf{K}(\bx)$, not on a correct specification of a full regression model in the features $\bx \in \mathcal{X}$. Misspecifying the weight matrix $W(\bx)$ can still affect efficiency, since the variance reduction relative to the IPS estimator is no longer guaranteed. Still, Proposition~\ref{prop:variance_bound} provides a bound to the variance. In Section~\ref{sec:numerical_part}, we show that, in practice, relying on the naive kernel matrix corresponding to the weight matrix defined in \eqref{eq:diagonal_weight_matrix} may be sufficient.

\section{Computation of optimal policy} \label{sec:policy_optimization}
In Sections~\ref{sec:estimation} and \ref{sec:kernelized_controls}, we described methods to determine the value of a new policy $\bar{\pi}$, given an observed learning sample recorded under a different policy $\widetilde{\pi}$. 
It remains to approximate the optimal policy $\bar{\pi}^*$, see \eqref{optimal control}. A distinctive feature of the off-policy framework is that the entire optimisation is performed offline on a fixed learning sample $\mathcal{L}_{\widetilde{\bpi}}$. This has two important practical consequences. First, any number of candidate policies can be evaluated and compared without additional data collection: the IPS and kernelized IPS estimators reuse the same experimental observations for every counterfactual evaluation, enabling exhaustive search over policy parameterisations at no marginal experimental cost. Second, the reliability of the selected policy can be verified on held-out data. Specifically, by partitioning the learning sample into a training set (used for policy optimisation) and a validation set (used for value estimation), we obtain an unbiased estimate of the optimised policy's value that is free of the selection bias inherent in evaluating a policy on the same data used to choose it. This train/validation split is analogous to cross-validation in supervised learning and provides a principled safeguard against overfitting to noise in the reward estimates.
In the following, we describe two alternatives.

\paragraph{Data-shared Lasso.}
The following method has been proposed by \cite{SP}, and it is based on \cite{Gross2016}.
We rewrite the kernelized IPS estimator as follows
\begin{equation*}
    \widehat{V}_{\rm K}(\bar{\pi}) 
    = \frac{1}{n}\sum_{i=1}^n R_i\, \frac{\left\langle\be_{A_i},\, \mathbf{K}_i(\bX_i)\, \bar{\be}_{\bar{\pi}(\bX_i)}\right\rangle }{\widetilde{\pi}_i(A_i\mid\bX_i)}
    = \frac{1}{n}\sum_{i=1}^n
    \sum_{\bar{a} \in \bar{\cal A}}
    \mathds{1}_{\{\bar{\pi}(\bX_i) = \bar{a}\}} \, V_{i,\bar{a}}\,,
\end{equation*}
where we define
\begin{equation*}
V_{i,\bar{a}} = R_i\, \frac{\left\langle\be_{A_i},\, \mathbf{K}_i(\bX_i)\, \bar{\be}_{\bar{a}}\right\rangle }{\widetilde{\pi}_i(A_i\mid\bX_i)}\,.
\end{equation*}
The goal is to remove the explicit dependence on $(A_i,R_i)$ from $V_{i,\bar{a}}$ by again considering a regression setting (i.e., we want to remove this dependence by exploiting a conditional expectation). This yields a policy optimization problem that depends only on the covariates $\bX_i$, allowing us to directly construct a decision rule. For this, $V_{i,\bar{a}}$ are approximated by a so-called data-shared Lasso regression. This approach is based on a linear regression in the risk factors $\bx \in \mathcal{X}$ defined by the parameters $\bw_0 \in \R^p$ and with action-specific deviations defined by the parameters $\bw_{\bar{a}} \in \R^p$ for actions $\bar{a} \in \bar{\mathcal{A}}$. This yields the conditional expectation approximation
\begin{equation}\label{data-shared Lasso}
V_{i,\bar{a}} \approx \bX_i^\top \left(\widehat{\bw}_{0}+ \widehat{\bw}_{\bar{a}}\right),
\end{equation}
where we regularize the two sets of parameters. Specifically, we solve
\begin{equation*}
    (\widehat{\bw}_0, \widehat{\bw}_{\bar{a}^{(1)}}, \dots, \widehat{\bw}_{\bar{a}^{(m)}}) = \underset{\bw_0, \bw_{\bar{a}}}{\arg\min} \sum_{i=1}^n \sum_{\bar{a} \in \bar{\cal A}} \left(V_{i,\bar{a}} - \bX_i^\top \left(\bw_{0}+ \bw_{\bar{a}}\right) \right)^2 + \tau \left( \lVert \bw_{0} \rVert_1 + \sum_{\bar{a} \in \bar{\cal A}} \gamma_{\bar{a}} \lVert \bw_{\bar{a}}\rVert_1 \right)\,,
\end{equation*}
where $\tau>0$ is a global regularization parameter and $\gamma_{\bar{a}}>0$ are additional regularization parameters. For details on the choice of $\gamma_{\bar{a}}>0$, see \cite{Gross2016}; if there is an equal number of observations for each action, \cite{Gross2016} suggest choosing $\gamma_{\bar{a}} = \frac{1}{\sqrt{m}}$, where $m$ is the number of actions in the new action space $\bar{\mathcal{A}}$.

Using approximation \eqref{data-shared Lasso}, we can remove the individual instance dependence, which allows us to approximate the policy value by
\begin{equation*}
    \widetilde{V}_{\rm K}(\bar{\pi})
    = \frac{1}{n}\sum_{i=1}^n \bX_i^\top
    \sum_{\bar{a} \in \bar{\cal A}}
    \mathds{1}_{\{\bar{\pi}(\bX_i) = \bar{a}\}} \left(\widehat{\bw}_{0}+ \widehat{\bw}_{\bar{a}}\right)\,.
\end{equation*}
We then approximate the optimal policy as
\begin{equation*}
    \bar{\pi}^{\text{DSL}}(\bx) 
    ~=~\underset{\bar{a}^\ast  \in \bar{\cal A}}{\arg \max}~
    \sum_{\bar{a} \in \bar{\cal A}}
    \mathds{1}_{\{\bar{a}^\ast = \bar{a}\}} \, \bx^\top  \widehat{\bw}_{\bar{a}}\,,\quad \bx \in \mathcal{X}\,.
\end{equation*}

\paragraph{Neural network approximation.}

We consider a neural network (NN) as a flexible approximator for the optimal pricing policy. Formally, a feedforward neural network defines a recursive sequence of transformations
\begin{equation*}
z^{(0)} = \bx, \quad 
z^{(\ell)} = \sigma^{(\ell)}\left(W^{(\ell)} z^{(\ell-1)} + b^{(\ell)}\right), \quad \ell = 1, \dots, N,
\end{equation*}
where $W^{(\ell)} \in \R^{n_{\ell-1} \times n_\ell}$ and $b^{(\ell)}\in \R^{n_\ell}$ are the weight matrix and bias vector of layer~$\ell$, which consists of $n_\ell$ neurons. The function $\sigma^{(\ell)}$ denotes an element-wise activation (e.g., ReLU or $\tanh$). 
The output layer $z^{(N)}$ is assumed to be a vector in $\mathbb{R}^m$, where $m$ was the number of actions in $\bar{\mathcal{A}}$.
The corresponding network policy $\bar{\pi}_\theta$ is then defined via a softmax activation on the output layer
\begin{equation*}
\bar{\pi}_\theta(\bar{a}^{(i)}\mid\bx) 
= 
\frac{\exp\left(z^{(N)}_{i}(\bx,\theta)\right)}
{\sum_{j = 1}^m \exp\left(z^{(N)}_{j}(\bx,\theta)\right)}\,, \quad 1 \leq i \leq m\,,
\end{equation*}
where $\theta = \{W^{(\ell)}, b^{(\ell)}\}_{\ell=1}^N$ denotes all trainable network parameters. 
The softmax transformation ensures that $\bar{\pi}_\theta(\cdot\mid\bx)$ defines a stochastic control on the new action space $\bar{\mathcal{A}}$.

The policy network is trained to maximize the estimated policy value
\begin{equation*}
\theta^* = \arg\max_\theta \, \widehat{V}_{\mathrm{K}}(\bar{\pi}_\theta)\,,
\end{equation*}
where $\widehat{V}_{\mathrm{K}}(\bar{\pi}_\theta)$ is the kernelized IPS estimator defined in Section~\ref{sec:kernelized_controls}, see formula \eqref{kernelized IPS2}.\\
Note that, so far, we have only defined the kernelized IPS estimator for deterministic policies. However, it naturally extends to stochastic policies by replacing the one-hot encoding vector $\bar{\be}_{\bar{\pi}(\bx)}$ with a vector of action probabilities $(\pi(\bar{a}^{(1)} \mid \bx), \dots, \pi(\bar{a}^{(m)} \mid \bx))^\top \in \R^m$ in \eqref{kernelized IPS2}.
\\
In practice, the maximization is performed using stochastic gradient ascent (or an adaptive optimizer such as Adam) on mini-batches of the learning sample $\mathcal{L}_{\widetilde{\bpi}}$.

Overall, this yields an adaptive policy that maps individual features~$\bx$ to optimal pricing actions.

In addition, since the output is a vector, $\bar{\pi}_\theta(\bar{a}\mid\bx)$ can be interpreted either as a stochastic policy or, by taking
\begin{equation*}
    \bar{\pi}^{\text{NN}}(\bx) = \arg\max_{\bar{a} \in \bar{\mathcal{A}}} \bar{\pi}_\theta(\bar{a}\mid\bx)\,,
\end{equation*}
as a deterministic control.\\

\section{Application to insurance pricing}\label{sec:numerical_part}
Consider an insurer offering travel insurance alongside airline tickets on a price comparison website. For each potential customer, the insurer observes characteristics $\bX$, such as the ticket price, travel duration, and booking lead time. Based on this information, the insurer chooses a pricing action $A = A(\bX)$, which adjusts the loading applied to a fair premium $P_{\mathrm{fair}}$ according to
\begin{equation*}
    P(\bX, A) = P_{\mathrm{fair}}(\bX)\big(1 + (1 + A)\,\lambda\big)\,,
\end{equation*}
where $\lambda > 0$ is the default profit loading. The customer then decides whether to purchase the insurance at price $P(\bX, A)$. If the contract is accepted, the insurer's reward $R$ is the resulting expected profit
\begin{equation*}
    R = P(\bX, A) - P_{\mathrm{fair}}(\bX)\,.
\end{equation*}
Otherwise, the profit is zero. One could also define the reward as the realized profit by replacing the fair premium with realized claims. Here, however, we assume that the idiosyncratic claims risk is diversified across the portfolio. Thus, the insurer’s reward depends both on the chosen premium and the customer’s price sensitivity. Our goal is to learn, from historical data, a pricing rule that maps customer characteristics to premium adjustments to maximize expected profit.

\subsection{Synthetic data generation}
\label{sec:data-generation}

To evaluate and compare the different off-policy value estimation and pricing optimization methods
introduced in Sections~\ref{sec:estimation}--\ref{sec:policy_optimization}, we generate a large-scale
synthetic data set that closely mimics a realistic travel insurance pricing environment.
Using synthetic data allows us to (i) control the underlying behavioral response to prices,
(ii) define a known ground truth for expected rewards, enabling objective benchmarking,
and (iii) ensure that the policy $\widetilde{\pi}$ is fully observed.

We simulate the learning sample
\begin{equation*}
    \mathcal{L}_{\widetilde{\bpi}} = (\bX_i, A_i, R_i)_{i=1}^n\,,
\end{equation*}
with $n$ observations.
All observations are generated under a single known, uniform policy, i.e., $\widetilde{\pi}_i = \widetilde{\pi}$, $1\leq i\leq n$, on the finite action space
\begin{equation*}
    \mathcal{A} = \{-20\%, -10\%, 0\%, +10\%, +20\%\}\,,
    \qquad \widetilde{\pi}(a \mid \bx) = \frac{1}{5}\,, \quad a \in \mathcal{A}, \;\bx \in \mathcal{X}\,.
\end{equation*}
The actions represent an increase/decrease of the default profit loading $\lambda$.

\paragraph{Covariates.}

For each policyholder $i\in \{1, \dots, n\}$, we generate a feature vector $\bX_i\in \R^p$
where the components follow the sampling rules in Table~\ref{tab:covariates}. We assume that all covariates are mutually independent. In the following sections, we will treat the country of destination as a latent variable. We intentionally exclude the destination variable to introduce controlled model misspecification, allowing us to study the robustness of different off-policy estimators.

\begin{table}[htbp]
    \centering
    \begin{tabular}{lll}
    \toprule
    Variable & Support & Sampling rule \\
    \midrule
    Ticket price & $[100,2000]$ & continuous uniform \\
    Lead time & $\{1,\dots,365\}$ & discrete uniform \\
    Number of Passengers & $\{1,\dots,5\}$ & discrete uniform \\
    Country of origin & 7 categories & discrete uniform \\
    Country of destination & 7 categories & discrete uniform \\
    Return trip & $\{\text{True},\text{False}\}$ & discrete uniform \\
    Trip duration & $\{1,\dots,30\}$ & discrete uniform\\
    \bottomrule
    \end{tabular}
    \caption{List of covariates for the travel insurance pricing environment.}
    \label{tab:covariates}
\end{table}

\paragraph{Premium adjustment.}
We assume that the insurer determines a baseline (actuarially fair) premium. For simplicity, we set this as a fixed share of the displayed ticket price $T_i$ (first component of $\bX_i$), i.e.,
\begin{equation*}
    P_{\mathrm{fair}}(\bX_i) = 0.1 \; T_i\,.
\end{equation*}
On top of this fair premium, the insurer charges a profit margin. We assume that, by default, this loading is given by a constant factor $\lambda > 0$. The pricing decision of the insurer then consists of adjusting this loading depending on the characteristics of the policyholder. This adjustment is captured by the action $A_i$, which scales the profit loading up or down.

More precisely, the charged premium $P(\bX_i, A_i)$ is given by
\begin{equation*}
    P(\bX_i, A_i) = P_{\mathrm{fair}}(\bX_i)\bigl(1 + (1 + A_i)\,\lambda\bigr)\,,
\end{equation*}
where $A_i < 0$ corresponds to a reduction and $A_i > 0$ to an increase in the loading relative to the default level $\lambda = 5\%$.

\paragraph{Conversion model.}
Customer decisions are driven by the offered premium. In particular, we assume that the probability of purchasing the insurance depends on the charged premium $P(\bX_i, A_i)$ through a price-sensitive demand model.

We model the conversion probability as a function of the policyholder characteristics $\bX_i$ and the pricing action $A_i$. Specifically, the policyholder is assumed to accept the offer with probability $p(\bX_i, A_i)$, which depends on an individual local price elasticity $E(\bX_i)$ via
\begin{equation*}
    p(\bX_i,A_i) = \min\left(1, \sigma\left(\bX_i^\top \balpha_1\right)
    \left( 1 + E(\bX_i)\,A_i\right)\right)\,,
\end{equation*}
for some weight vector $\balpha_1 \in \R^p$ and the logistic sigmoid function $\sigma$. The truncation at $1$ ensures valid probabilities and introduces mild nonlinearities near the boundaries of the action space. For $A_i = 0$, the baseline conversion probability for the premium 
\begin{equation*}
    P(\bX_i, 0) = P_{\mathrm{fair}}(\bX_i)(1 + \lambda)
\end{equation*}
is given by the term $\sigma\left(\bX_i^\top \balpha_1\right)$.
The local price elasticity $E(\bX_i)$ captures how sensitive the policyholder is to price changes, with more negative values corresponding to higher price sensitivity. It is defined as
\begin{equation*} \label{eq:elasticity}
    E(\bx) = - \min\left( \exp(\bx^\top \balpha_2 + h(\bx)), 4 \right)\,,\quad \bx \in \mathcal{X}\,,
\end{equation*} 
for a weight vector $\balpha_2\in \R^p$. Note that $E(\bx) \le 0$ ensures that demand decreases as prices increase. The function $h(\bx)$ collects higher-order terms in the covariates and is defined as
\begin{equation}\label{eq:higher_order_terms}
    h(\bx) =
    \balpha_3^\top
    \begin{pmatrix}
        x_1^3 \\
        x_1 x_2 \\
        x_1 x_3 \\
        x_3 x_6
    \end{pmatrix}\,,
\end{equation}
where $x_1$ denotes the ticket price, $x_2$ the lead time, $x_3$ the number of passengers, and $x_6$ the return-trip indicator, while $\balpha_3 \in \R^4$ is another weight vector.
We will investigate the effect of this higher-order term $h(\bx)$ on the performance of different estimators and policy optimization approaches in the following.

Let $C_i$ denote the conversion outcome for policyholder $\bX_i$, which equals $1$ if the policyholder purchases the insurance and $0$ otherwise. It is the realization of a Bernoulli variable with probability $p(\bX_i, A_i)$. 

\paragraph{Rewards.}

Given the conversion outcome $C_i$, the reward $R_i$ from policyholder $\bX_i$ for the insurer corresponds to the expected profit. That is, if the policyholder accepts the offer, the insurer expects to earn the difference between the charged premium and the fair premium, and otherwise, the reward is zero. Hence,
\begin{equation*}
    R_i = C_i \bigl(P(\bX_i, A_i) - P_{\mathrm{fair}}(\bX_i)\bigr)
    = C_i\, P_{\mathrm{fair}}(\bX_i)\,(1 + A_i)\,\lambda\,.
\end{equation*}

For each observation $i$, we additionally compute the expected rewards for all possible actions $a\in \mathcal{A}$,
\begin{equation} \label{eq:true_expected_rewards}
    \varrho(\bX_i, a) = p(\bX_i, a)\, P_{\mathrm{fair}}(\bX_i)\,(1 + a)\,\lambda\,.
\end{equation}
This information is used for counterfactually validating new policies. Note that the premium for policyholder $i$ is linear in the action $A_i$. The same holds approximately for the conversion probability $p(\bX_i, A_i)$, which is, however, capped at 1. Thus, the expected reward is approximately quadratic in the action $A_i$. We will use this information in the construction of our kernelized IPS estimator. For sufficiently large discounts or surcharges, however, the truncations introduce nonlinearities that break this quadratic relationship, which will affect extrapolation accuracy, see Section~\ref{sec:extrapolation}, below.

\subsection{Predict-then-optimize} \label{sec:PTO}
As a baseline for optimal pricing, we consider a classical predict-then-optimize (PTO) strategy, see, e.g., \cite{Chen2022} and \cite{Elmachtoub2022}. In this approach, the demand or conversion behavior is first estimated using a parametric regression model, and the optimal pricing decision is obtained by maximizing the resulting predicted reward. In contrast to the two methods introduced in Section~\ref{sec:policy_optimization}, this approach relies on an explicit structural model for customer response.\\

We model the conditional conversion probability $p(\bx,a)$ for $\bx \in \mathcal{X}$ and
$a \in \mathcal{A}$ using a logistic regression model. Specifically, we assume that the
conversion indicator satisfies
\begin{equation*}
    C \mid_{\bX=\bx, A=a}\, \sim \text{Bernoulli}(p(\bx,a)).
\end{equation*}
The conversion probability is approximated by
\begin{equation*}
    \widehat p(\bx,a) = g^{-1}\!\big(\eta(\bx,a)\big),
\end{equation*}
where $g : (0,1) \to \mathbb{R}$ denotes the logit link function
\begin{equation*}
    g(\mu) = \log\!\left(\frac{\mu}{1-\mu}\right).
\end{equation*}
We consider a score $\eta$ that contains a quadratic influence in the actions, i.e. 
\begin{equation*}
    \eta(\bx,a)
    =
    \boldsymbol{\phi}_1^\top
    \begin{pmatrix}
        1 \\
        a \\
        a^2
    \end{pmatrix}
    +
    \boldsymbol{\phi}_2^\top \bx 
    +
    \boldsymbol{\phi}_3 \, a\, \bx^\top \,,
\end{equation*}
where $\boldsymbol{\phi}_1 \in \mathbb{R}^3$ represents the polynomial effect of the action, $\boldsymbol{\phi}_2 \in \mathbb{R}^p$ captures the main effects of the covariates, and $\boldsymbol{\phi}_3 \in \mathbb{R}^p$ represents action-covariate interactions.
Note that this regressor does not capture the higher-order interactions $h(\bx)$, as selected in \eqref{eq:higher_order_terms}.

Given an estimate $\widehat{p}(\bx,a)$ of the conversion probability, the expected reward of performing action $a$ is approximated by the DM estimator
\begin{equation*}
    \widehat{\varrho}^{\,\text{DM}}(\bx,a)
    = \widehat{p}(\bx,a)\big(P(\bx,a) - P_{\mathrm{fair}}(\bx)\big) \,,
\end{equation*}
where $P(\bx,a)$ denotes the adjusted premium associated with action $a \in \mathcal{A}$. 
While the historical data are generated on the original action space $\mathcal{A}$, the policy optimization procedures introduced below may operate on a new action space $\bar{\mathcal{A}}$. This requires that the expected reward function $\varrho(x,a)$ extends to actions outside the observed set.
The resulting policy is obtained by selecting the action that provides the maximal estimated reward, i.e.,
\begin{equation*}
\bar{\pi}^{\text{PTO}}(\bx)
~=~
\underset{\bar{a} \in \bar{\mathcal{A}}}{\arg\max}\;
\widehat{\varrho}^{\,\text{DM}}(\bx,\bar{a})\,.
\end{equation*}
The advantage of this two-stage procedure lies in its interpretability and computational simplicity, as standard actuarial tools such as GLMs can be used to model the demand. However, the performance depends on the accuracy of the specified regression model. Model misspecification may lead to biased reward estimates and, therefore, to suboptimal pricing decisions. Crucially, this may propagate to the estimated policy value if the same demand model is reused for estimating the value
\begin{equation*}
    \widehat{V}_{\rm DM} (\bar{\pi}^{\text{PTO}}) =  \frac{1}{n} \sum_{i=1}^n \underset{\bar{a} \in \bar{\mathcal{A}}}{\max}\;
    \widehat{\varrho}^{\,\text{DM}}(\bX_i,\bar{a})\,.
\end{equation*}
We will return to this issue in Section~\ref{subsec:policy_optimization}.
The PTO approach serves as a natural benchmark for the kernelized off-policy optimization methods introduced in Section~\ref{sec:policy_optimization} and studied below.

\subsection{Variance reduction of the kernelized IPS estimator} 

A primary motivation for introducing the kernelized IPS estimator $\widehat V_{\mathrm{K}}$ is variance reduction compared to the classical IPS estimator $\widehat V_{\mathrm{IPS}}$, see Propositions~\ref{prop:variance_bound}--\ref{prop:min_variance}. This section empirically verifies these theoretical results.

We consider a fixed deterministic policy $\pi_0$ corresponding to the constant action $a=0$, i.e., $\pi_0(\bx)=0$ for all $\bx\in\mathcal X$. For varying sample sizes $n$, we generate learning samples $\mathcal L_{\widetilde{\bpi}}$ and evaluate the estimators on these samples. Since the synthetic data generation process provides access to the true expected rewards $\varrho(\bx,a)$ via \eqref{eq:true_expected_rewards} for all policyholder features $\bx\in\mathcal X$ and actions $a\in\mathcal A$, the empirical policy value $\widehat{V}_n(\pi)$ is known, which allows us to compute the root mean square error (RMSE) of the estimators directly.

The kernel matrix is constructed according to Section~\ref{sec:kernelized_controls} using the regression assumption~\eqref{linear regression ansatz} with a quadratic basis in the action variable, i.e., $q=2$ with $f_1(a)=a$ and $f_2(a)=a^2$, $a\in\mathcal A$. This choice reflects the structure of the simulation environment, where the expected rewards are approximately quadratic in the actions as discussed in Section~\ref{sec:data-generation}.

The left panel of Figure~\ref{fig:variance_reduction} illustrates the empirical RMSE of the DM estimator $\widehat V_{\mathrm{DM}}$, the classical IPS estimator $\widehat V_{\mathrm{IPS}}$, and the variance-optimal kernelized IPS estimator $\widehat V_{\mathrm{K}}$ as a function of the sample size when evaluating the constant policy $\pi_0$. We observe that the kernelized IPS estimator exhibits substantially lower RMSE than the IPS estimator across all sample sizes. This behavior empirically confirms Corollary~\ref{cor:lowest_bound}, which predicts a variance reduction since the kernel aggregates information across neighboring actions. In contrast, the IPS estimator only uses observations for which the realized action $A_i$ matches the action chosen by a target policy $\pi(\bX_i)$, $1\leq i \leq n$, which leads to a higher variability. The kernelized IPS estimator achieves approximately the same RMSE as the IPS estimator does with twice the number of samples. At the same time, beyond a sample size of $n = 50\,000$, the RMSE of the DM estimator does not significantly decrease further. This behavior is driven by misspecification of the global regression model $\widehat{\varrho}^{\,\mathrm{DM}}$. In particular, the model fails to capture the higher-order interaction term $h(\bx)$ and omits the flight destination, which acts as a latent variable. As a result, the estimated rewards are systematically biased.

In addition to the variance-optimal kernel derived in Proposition~\ref{prop:min_variance}, we also consider a naive kernelized IPS estimator based on the diagonal weight matrix~\eqref{eq:diagonal_weight_matrix}. The comparison between the naive and the variance-optimal kernelized IPS estimators is shown in the right panel of Figure~\ref{fig:variance_reduction}. Each point corresponds to one simulation run with a fixed sample size $n = 500\,000$, and the two axes display the corresponding kernelized IPS estimates obtained under the two kernel constructions, each for the value of the constant policy $\pi_0$. The estimates lie almost exactly on the diagonal, indicating that both approaches produce essentially identical results in this setting.

Note that, conditional on conversion, the reward associated with an action $a \in \mathcal{A}$ is given by
\begin{equation*}
    P_{\mathrm{fair}}(\bX_i)\,(1 + a)\,\lambda\,.
\end{equation*}
Thus, the effect of the action on the reward is relatively moderate in our setting, as it enters only through the multiplicative factor $(1+a)$. In particular, the variation across actions is limited to a narrow range between $0.8$ and $1.2$ in our specification. As a consequence, differences in rewards across actions are relatively small, and the benefit of using a variance-optimal kernel over a simpler (naive) kernel is less pronounced. In settings where the action has a stronger impact on the reward, we expect the variance reduction achieved by the optimal kernel to be more substantial.

From a computational perspective, replacing the variance-optimal kernel matrix with the naive kernel matrix is particularly relevant. The naive kernelized IPS estimator avoids the need to estimate conditional reward moments or covariance structures when constructing the kernel matrix. In contrast, constructing the variance-optimal kernel requires estimating the conditional covariance matrix $\Sigma(\bx)$ and computing a policyholder-specific kernel matrix, which involves forming the matrix $D^\top \Sigma(\bx)^{-1}D$ and its inversion.
The naive kernel, on the other hand, allows reusing the same matrix across policyholders sharing identical action propensities and therefore provides a substantial scalability advantage in large-scale applications.

\begin{figure}[htbp]
	\centering
    \includegraphics[width=0.98\textwidth]{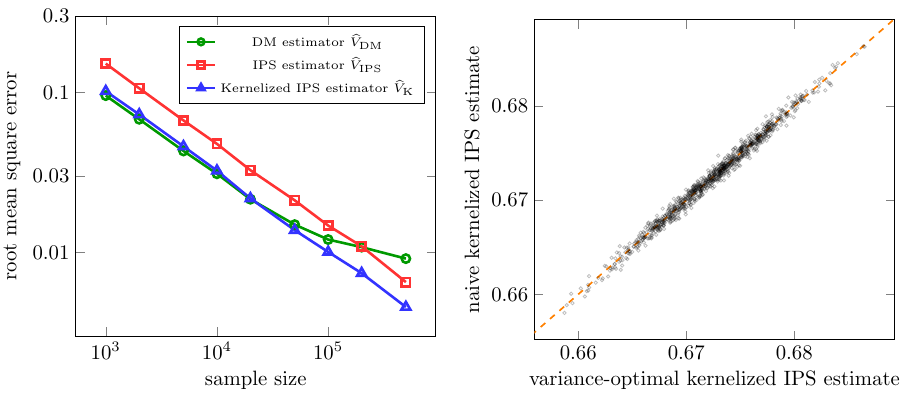}
    \caption{Left panel: Empirical RMSE of the IPS estimator, the variance-optimal kernelized IPS estimator, and the DM estimator for the constant policy $\pi_0(\bx) = 0$, shown as a function of the sample size on the log–log scale.\\
    Right panel: Scatter plot of the naive kernelized IPS estimates versus the variance-optimal kernelized IPS estimates, each for the value of the constant policy $\pi_0(\bx) = 0$, across different simulations with $n=500\,000$ samples each. The orange dashed line represents the identity line, indicating equality between the two estimators.}
    \label{fig:variance_reduction}
\end{figure}

\subsection{Extrapolation with the kernelized IPS estimator} \label{sec:extrapolation}

A further advantage of the kernelized IPS estimator is its ability to evaluate policies defined on action spaces that extend beyond the actions observed in the historical data. Recall that the learning sample is generated on the action space
\begin{equation*}
    \mathcal{A}=\{-20\%, -10\%, 0\%, +10\%, +20\%\}\,.
\end{equation*}
In practice, however, insurers may wish to evaluate pricing strategies that involve stronger discounts or surcharges, or a finer resolution of possible price adjustments.

To study this behavior, we consider an extended action space defined on a dense grid,
\begin{equation*}
    \bar{\mathcal{A}} = \{-30\%, -29\%, -28\%, \ldots, 28\%, 29\%, 30\%\}\,,
\end{equation*}
which contains actions both outside the historical range and intermediate adjustments that were not offered under policy $\widetilde{\pi}$. This grid can be viewed as a discretized approximation of a continuous pricing adjustment interval.

The kernelized IPS estimator introduced in Section~\ref{sec:kernelized_controls} naturally accommodates such settings. Because the kernel matrix is constructed using a regression structure in the action variable, it allows information from the observed actions to be smoothly extrapolated to nearby actions in the extended space. In contrast, the classical IPS estimator cannot evaluate policies on $\bar{\mathcal{A}}$, since it requires exact matches between the actions taken by the target policy and those observed in the historical data.

For each action $\bar{a} \in \bar{\mathcal{A}}$, we compute the value of the constant policy $\bar\pi_{\bar{a}}(\bx)=\bar{a}$ using the kernelized IPS estimator based on learning samples of size $n=100\,000$. Since the synthetic data generation process provides access to the true expected rewards $\varrho(\bx,\bar{a})$ in \eqref{eq:true_expected_rewards}, we again measure accuracy via the RMSE across repeated simulations.

Figure~\ref{fig:extrapolation} displays the empirical RMSE of the kernelized IPS estimator for different constant policies. The estimator remains stable across the central region of the action space and produces reliable estimates even for actions slightly outside the historical support. However, the estimation error increases more noticeably when moving further beyond the interval $[-20\%,20\%]$.

This behavior can be explained by the structure of the data generation process described in Section~\ref{sec:data-generation}. Recall that the premium $P(\bx,a)$ is linear in the action $a$, while the purchase probability
\begin{equation*}
    p(\bx,a)=\min\!\left(1,\;\sigma\!\left(\bx^\top\balpha_1\right)(1+E(\bx)a)\right)
\end{equation*}
is approximately linear in $a$ only when the truncation at $1$ is inactive. As discussed in Section~\ref{sec:data-generation}, this truncation introduces nonlinearities in the reward structure once the actions become sufficiently large in magnitude. Consequently, the expected reward $\varrho(\bx,a)$ deviates from the approximately quadratic relationship in $a$ that motivates the kernel construction in Section~\ref{sec:kernelized_controls}.

As a result, extrapolation becomes less accurate for large discounts or surcharges, where these caps are more likely to be active. Nevertheless, the results show that the kernelized framework can effectively interpolate between observed actions and extrapolate moderately beyond them, allowing the evaluation of pricing strategies on a much finer and broader set of possible price adjustments.

\begin{figure}[htbp]
	\centering
    \includegraphics[width=0.6\textwidth]{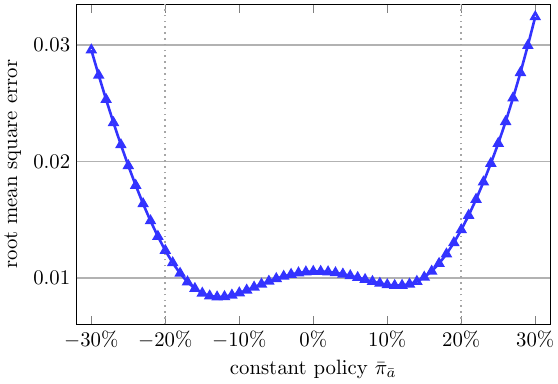}
    \caption{Empirical RMSE of the kernelized IPS estimator for constant policies $\bar\pi_{\bar a}(\bx)=\bar a$, $\bar a \in \bar{\mathcal{A}}$, across different simulations of sample size $n=100\,000$.}
    \label{fig:extrapolation}
\end{figure}

\subsection{Policy optimization} \label{subsec:policy_optimization}

We now turn to the empirical performance of the policy optimization methods introduced in Section ~\ref{sec:policy_optimization}. In particular, we compare the data-shared Lasso (DSL) and the neural network (NN) approaches to the predict-then-optimize (PTO) benchmark.

\paragraph{Implementation of data-shared Lasso and neural network policies.}

We implement both the DSL and the NN policies using standardized input features, with categorical variables encoded via one-hot representations. The neural network architecture consists of two hidden layers with $32$ units each and $\mathrm{ReLu}$ activation functions.

Hyperparameters are tuned via automated search over predefined ranges. For DSL, this includes the regularization parameter and interaction penalty. For the NN, we tune network depth, width, learning rate, and batch size using randomized search.

\paragraph{Benchmark.}

Since the simulation environment provides the expected reward function $\varrho$, we determine the optimal policy
\begin{equation*}
    \bar{\pi}^\ast(\bx) = \underset{\bar{a} \in \bar{\mathcal{A}}}{\arg\max} \, \varrho(\bx,\bar{a})\,.
\end{equation*}
We compute its empirical policy value $\widehat{V}_n(\bar{\pi}^\ast)$, as defined in \eqref{eq:empirical_value}, on a \textit{reference data set} consisting of $n=1\,000\,000$ policyholders.

We generate 20 learning samples $\mathcal{L}_{\widetilde{\bpi}}$, containing $n=1\,000\,000$ observations each, used to fit 20 policies $\bar\pi$ per method. The training procedure of the neural network is stochastic, and results depend on the initialization of the trainable network parameters $\theta$. This is why for every learning sample, we fit 5 networks with different random initializations. Ultimately, we select the network and the training epoch that produces the highest kernelized IPS estimate on a subset of the learning sample that has not been used for training. This use of a held-out evaluation set is a direct benefit of the offline framework: because the kernelized IPS estimator is unbiased and can be computed on any subset of data collected under the logging policy, it provides a reliable, unbiased assessment of each candidate policy's performance without requiring deployment.

For each resulting policy $\bar\pi$, we evaluate its performance through the relative gap in policy value on the reference data set with respect to the optimal policy,
\begin{equation*}
    \frac{\widehat{V}_n(\bar\pi)-\widehat{V}_n(\bar\pi^\ast)}{\widehat{V}_n(\bar\pi^\ast)}\,,
\end{equation*}
Since $\widehat{V}_n(\bar\pi^\ast)$ represents the maximal achievable expected reward, this quantity measures the loss in expected reward relative to the optimal pricing strategy.

\paragraph{Results.} The top left panel in Figure~\ref{fig:relative_gap} illustrates the empirical relative gap to the optimal policy value $\widehat{V}_n(\bar{\pi}^\ast)$ for the policy determined via DSL, for the NN policy, and for the GLM-based PTO policy. Each gray bar shows the average gap obtained over all learning samples, and the error bars represent one standard deviation across these runs.

All approaches produce policies whose values are close to the optimal benchmark. The NN policy consistently achieves the smallest performance gap. This can be attributed to the ability of the NN to capture the higher-order interactions $h(\bx)$ that influence the expected reward in our simulation environment. In contrast to the NN policy, the PTO approach and the DSL policy are more sensitive to model misspecification since they are restricted to (generalized) linear relations in the features $\bx \in \mathcal{X}$. Overall, the DSL approach provides an interpretable linear structure with action-specific deviations, while the NN offers greater flexibility at the cost of reduced interpretability and substantially higher computational effort.

As described in Section~\ref{sec:PTO}, if we use the DM estimator to estimate the value of the PTO policy via
\begin{equation*}
    \widehat{V}_{\rm DM} (\bar{\pi}^{\text{PTO}}) =  \frac{1}{n} \sum_{i=1}^n \underset{\bar{a} \in \bar{\mathcal{A}}}{\max}\;
    \widehat{\varrho}^{\,\text{DM}}(\bX_i,\bar{a})\,,
\end{equation*}
we tend to overestimate the true value. This is because the estimated expected rewards $\widehat{\varrho}^{\,\text{DM}}$ are used for determining the PTO policy as well as estimating its value. Since this estimator is inevitably noisy, the maximization step introduces a selection bias, which favors actions with positive estimation errors and thereby leads to an upward bias in the estimated policy value. This effect can be seen in the bottom left panel of Figure~\ref{fig:relative_gap}.

\begin{figure}[htbp]
    \centering
    \includegraphics[width = 0.98\textwidth]{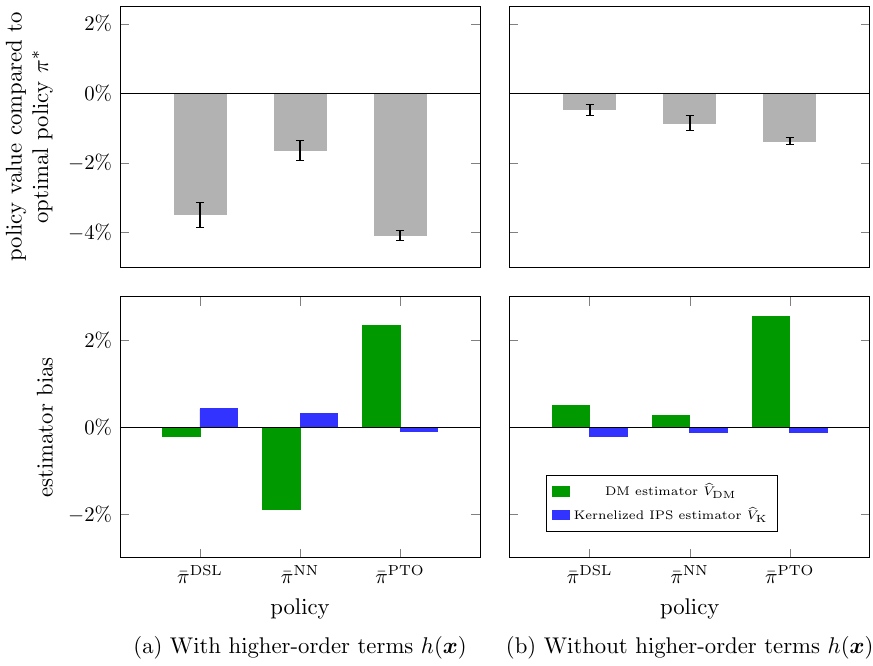}
    \caption{Top row: Mean relative gap in policy value to the optimal policy $\pi^\ast$ for the learned policies across 20 repeated samples of the training data. Gray bars show the mean gap, and black error bars indicate one standard deviation.\\
    Bottom row: Bias of the policy value estimators for the learned policies. Values above zero indicate overestimation of policy performance, while values below zero indicate underestimation.\\
    In the left column, the data-generating process includes the higher-order interaction term $h(\bx)$ in the price elasticity, as selected in \eqref{eq:higher_order_terms}, whereas in the right column this term is excluded.}
    \label{fig:relative_gap}
\end{figure}

To assess the sensitivity of the results to the complexity of the data generation process, we repeated the numerical experiments using a simplified specification without the higher-order interaction term $h(\bx)$ in the elasticity model \eqref{eq:elasticity}. In this setting, the expected reward is closer to the linear structure in the features $\bx \in \mathcal{X}$ implicitly assumed by the DSL and PTO policies. As anticipated, the performance differences between the methods become negligible as shown in the right panels of Figure~\ref{fig:relative_gap}. Crucially, however, the DM estimator overestimates the policy value of the PTO policy even when no higher-order term $h(\bx)$ is included. The kernelized IPS estimator only shows little noise for all policies. These results strongly suggest that the expected reward estimator $\widehat{\varrho}^{\,\text{DM}}$ used to calibrate the PTO policy should not be chosen to estimate its value.

\section{Conclusion}

This paper studies insurance pricing from the perspective of off-policy evaluation and control. By viewing pricing decisions as actions taken on policyholders with observable characteristics, we connect demand-sensitive pricing problems with tools from reinforcement learning and counterfactual inference. This perspective allows pricing strategies to be evaluated and optimized using historical data collected under previous pricing policies, without the need for new experiments.

A central contribution of the paper is the introduction of a kernelized inverse propensity score estimator that exploits structure in the action space. The estimator smooths information across neighboring actions while preserving unbiasedness under a mild regression assumption. This leads to a substantial reduction in variance relative to the classical inverse propensity score estimator and allows policies defined on new action spaces to be evaluated using historical data. We further characterize the variance-optimal kernel and discuss computationally efficient approximations that perform similarly in practice.

Building on these value estimates, we study two approaches for policy optimization: an interpretable data-shared Lasso formulation and a flexible neural network parameterization. In a controlled synthetic travel insurance environment, both approaches produce pricing rules that perform close to the optimal benchmark. The neural network policies achieve slightly higher policy values on average, while the data-shared Lasso provides a more transparent structure and performs similarly to the classical predict-then-optimize approach.

Overall, the results highlight the potential of off-policy methods for actuarial pricing problems where demand responses to price changes play an important role. The proposed framework cleanly separates experimentation from deployment: a single randomised data-collection phase supports offline evaluation of an arbitrary number of counterfactual pricing strategies, and the unbiasedness of the IPS-based estimators allows the insurer to validate the chosen policy on held-out experimental data before committing to production. This provides a principled, low-risk pathway from data collection to optimised pricing, while avoiding strong structural assumptions about customer behaviour. In addition, the kernelized inverse propensity score estimator enables interpolation between observed actions and moderate extrapolation to previously unseen price adjustments. As illustrated in our numerical experiments, the accuracy of such extrapolation naturally depends on how well the assumed local regression structure approximates the true reward function.

Several directions for future research appear promising. Extending the framework to sequential decision problems could allow insurers to study dynamic pricing strategies over the customer lifecycle. Incorporating richer behavioral models or market interactions may further improve policy performance in competitive environments. Finally, applying the proposed methods to real insurance portfolios would provide valuable insights into the practical impact of data-driven pricing strategies.

\bibliographystyle{apalike}
\bibliography{bibliography.bib}

\end{document}